\documentclass{article}

\usepackage{arxiv}

\usepackage{url,amsmath,amsfonts,amssymb}
\usepackage{caption}
\usepackage{subcaption}
\usepackage{algorithm}
\usepackage{algorithmic}
\usepackage{graphicx}
\usepackage{listings}
\usepackage[T1]{fontenc}
\usepackage{enumerate}   
\usepackage{color}
\usepackage{relsize}
\usepackage{multirow}
\usepackage[utf8]{inputenc}
\usepackage{booktabs}       % professional-quality tables
\usepackage{nicefrac}       % compact symbols for 1/2, etc.
\usepackage{microtype}      % microtypography
\usepackage[toc,page]{appendix}

\numberwithin{equation}{section}

\DeclareMathOperator*{\argmax}{arg\,max}

\DeclareMathOperator*{\mL}{\mathcal L}
\DeclareMathOperator*{\mS}{\mathcal S}

\def \endprf{\hfill {\vrule height6pt width6pt depth0pt}\medskip}

\newcommand{\BlackBox}{\rule{1.5ex}{1.5ex}}  % end of proof
\newenvironment{proof}{\par\noindent{\bf Proof\ }}{\hfill\BlackBox\\[2mm]}

\newtheorem{theorem}{Theorem}
\newtheorem{lemma}[theorem]{Lemma}

\title{HierLPR: Decision making in hierarchical multi-label classification with local precision rates}
\author{Christine Ho\thanks{Equally contributed.}\\
       Department of Statistics\\
       University of California\\
       Berkeley, CA, USA\\
       \texttt{christineho@berkeley.edu}\\
       \And
       Yuting Ye\footnotemark[1]\\
       Division of Biostatistics\\
       University of California\\
       Berkeley, CA, USA\\
       \texttt{yeyt@berkeley.edu} \\
       \And
       Ci-Ren Jiang\\       
       Institute of Statistical Science\\
       Academia Sinica\\
       Taipei, Taiwan\\
       \texttt{cirenjiang@stat.sinica.edu.tw} \\
       \And
       Wayne Tai Lee \\
       Department of Statistics\\
       University of California, Berkeley\\
       Berkeley, CA, USA\\
       \texttt{lwtai@stat.berkeley.edu}\\
       \And
       Haiyan Huang\thanks{Corresponding author.}\\
       Department of Statistics\\
       University of California\\
       Berkeley, CA, USA\\
       \texttt{hhuang@stat.berkeley.edu}\\
     }

\begin{document}
\maketitle

\begin{abstract}
  In this article we propose a novel ranking algorithm, referred to as HierLPR, for the multi-label classification problem when the candidate labels follow a known hierarchical structure. HierLPR is motivated by a new metric called eAUC that we design to assess the ranking of classification decisions. This metric, associated with the hit curve and local precision rate, emphasizes the accuracy of the first calls. We show that HierLPR optimizes eAUC under the tree constraint and some light assumptions on the dependency between the nodes in the hierarchy. We also provide a strategy to make calls for each node based on the ordering produced by HierLPR, with the intent of controlling FDR or maximizing F-score. The performance of our proposed methods is demonstrated on synthetic datasets as well as a real example of disease diagnosis using NCBI GEO datasets. In these cases, HierLPR shows a favorable result over competing methods in the early part of the precision-recall curve.
\end{abstract}

% keywords can be removed

\section{Introduction}\label{sec:hmc_intro}
Hierarchical multilabel classification (HMC) refers to the classification problem in which instances are assigned labels of multiple classes, and these classes follow one or more paths along a tree or directed acyclic graph (DAG) structure (Figure \ref{fig:notation_graph}). Typically, the predictions are not required to reach the leaf-node level. In computational biology and medicine, a typical application of HMC is the diagnosis of disease using genomics data. The disease labels are often derived from the Unified Medical Language System (UMLS), a well-known biomedical vocabulary organized as a directed acyclic graph \citep{huang2010}. Other common applications of HMC include the assignments of genes to categories from the Gene Ontology DAG, and the categorization of proteins along the MIPS FunCat rooted tree \citep{alves2010, barutcuoglu2006, blockeel2006decision, clare2003, kiritchenko2005, valentini2009, valentini2011}. Outside of computational biology, HMC is used for problems like the classification of text documents, music categorization, and image recognition, all of which tend to have labels that follow hierarchical structures \citep{rousu2006kernel, kiritchenko2006learning, mayne2009}.

Solutions to the HMC problem fall into three categories: flat, local, and global classification \citep{silla2011survey}. Flat classifiers generate predictions for the leaf nodes and then propagate any positive calls at the leaf-level to the their ancestors in order to respect the hierarchy \citep{barbedo2007automatic, silla2011survey}. One downside of this type of methods is that it does not take advantage of information from non-leaf level classes. Local classification uses a two-stage process: first, initial predictions over the nodes in the graph are generated from multiple classifiers, where each classifier has been trained separately for each node or a group of nodes in the graph; then, in the second stage, an adjustment is made to the initial predictions to generate final decisions that respect the hierarchical relationships among the classes. Local classifiers enjoy tremendous popularity because they are flexible and efficient, as demonstrated in \cite{koller1997, wu2005,holden2005hybrid,silla2009novel,gauch2009training}, to name a few. An open question in local classification is how to perform this second stage adjustment: it is common for second-stage adjustments to suffer from a blocking problem, where the first-stage decisions at the top of the hierarchy essentially carry too much weight. It is challenging to create rules that will allow strong positive signals from lower down in the hierarchy to override false negatives at the top. There are several tentative proposals, but they either lack theoretical support \citep{sun2001}, or do not scale well with the size of the graph and can incur numerical underflow issues \citep{barutcuoglu2006}. Unlike local classifiers, global classifiers simultaneously make predictions for the graph rather than on a node by node basis. \cite{vens2008} argued that global classification methods have the potential to perform better than local classification methods because they demand fewer decision rules overall. However, in practice, most global classifiers suffer from scalability issues and are often more computationally demanding because they require optimizing over the entire space of possible decisions. The current state-of-the-art method of HMC is a global classifier called clusHMC and its variants \citep{blockeel2002, blockeel2006decision, vens2008}, which are based on predictive clustering trees (PCTs). These techniques adapt decision tree learners to the HMC case for both tree- and DAG-structured hierarchies.

In this article, we propose a ranking strategy called \textit{Hierarchical multi-label classification with Local Precision Rates} (HierLPR), which can be used at the second stage of a local classification approach --- it produces a hierarchically consistent ranking of the instances based on their first-stage prediction scores. HierLPR is motivated by a novel metric called early-call-driven area under curve (eAUC) that we design to assess the ranking of classification decisions. This metric is closely related to the area under the hit curve and local precision rate (LPR) \citep{jiang2014}, and assigns more importance to the decisions made in the early stage than those made later. The optimization of eAUC under the hierarchy constraint forces the strong positive nodes from lower down to help sort those seated at the top of hierarchy, which alleviates the blocking issue for HierLPR to some extent. We evaluate the performance of HierLPR by comparing it with a current state-of-the-art method for global classification, ClusHMC. In our simulation studies, HierLPR has comparable performance to ClusHMC. When applied to the Gene Expression Omnibus (GEO) and the UMLS datasets for disease diagnosis \citep{huang2010}, HierLPR outperforms ClusHMC. Finally, in addition to providing a sorted list produced by HierLPR, we provide a strategy for selecting a cutoff on the ranking so that the instances ahead of the cutoff are taken as positive and others are negative. As the simulation study shows in Section \ref{sec:cutoff_selection}, we can control the precision rate or optimize F-measure by using the selected cutoff.

The rest of the paper is organized as follows. In Section \ref{sec:notation_concept}, we introduce the notation used throughout the paper, and review the concepts of LPR and hit curve. In Section \ref{sec:methods}, we define the novel metric eAUC to evaluate the classification performance, which originates from the hit curve. Next, we introduce and analyze HierLPR for performing the second-stage decision adjustment of a local classifier. Faster implementations of HierLPR are discussed in Section \ref{sec:faster_version}. We assess the performance of HierLPR on synthetic data and the disease diagnosis data used in \cite{huang2010} in Section \ref{sec:simulation} and \ref{sec:real_data} respectively. In Section \ref{sec:cutoff_selection} we introduce the strategy for cutoff selection and demonstrate its power on the synthetic data. Finally, we conclude this article in Section \ref{sec:hmc_conclusion}.

\section{Notation and related concepts} \label{sec:notation_concept}
In this section, we formulate the question of our interest and introduce the notation used throughout this article. The notation is based on the tree structure, the focus of this article. Then, we briefly review the concepts of LPR and the hit curve, both of which are key concepts of our procedure. \\

\subsection{Notation} \label{sec:notation}
\begin{figure}[h]
  \centering
  \includegraphics[width=0.6\textwidth]{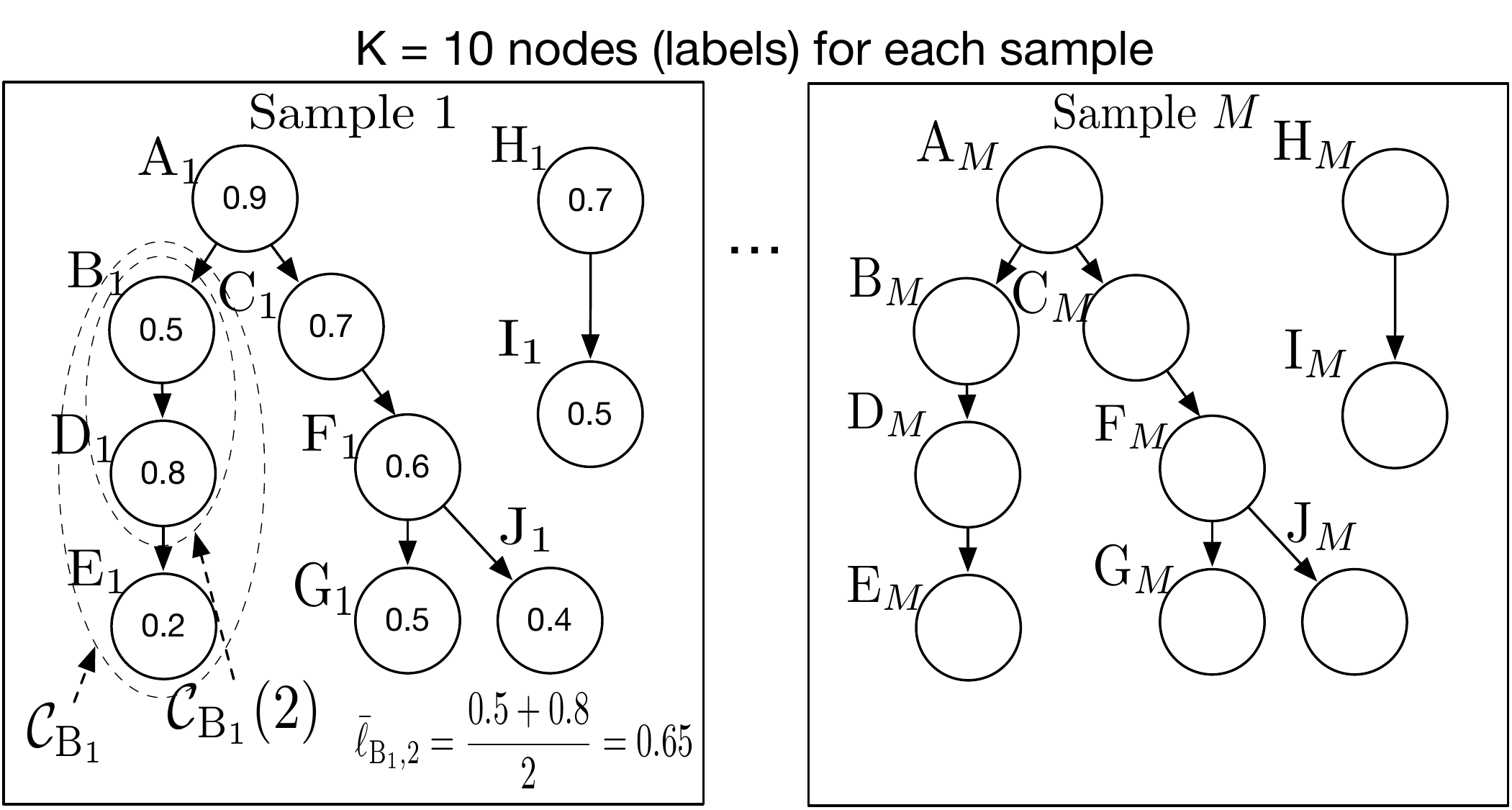}
  \caption{Graphical illustration of the notation. Each sample has $10$ candidate labels (nodes), and these labels follow a forest hierarchy. As shown in Sample 1, the numbers inside the nodes are the associated node LPRs (scores) for the this sample. $\mathcal{C}_{\text{B}_1}$ (the chain in the outer dash ellipse) is a sub-chain consisting of the three nodes starting from node $\text{B}_1$. $\mathcal{C}_{\text{B}_1}(2)$ (the chain in the inner dash ellipse) is a sub-chain consisting of the first two nodes of $\mathcal{C}_{\text{B}_1}$, and $\bar{\ell}_{\text{B}_1, 2}$ is the averaging LPR (score) of these two nodes.}
  \label{fig:notation_graph}
\end{figure}

In this article, we differentiate the terms ``sample'' and ``instance'' --- a ``sample'' is used to represent an individual or an object while an ``instance'' particularly refers to a (sample, class) or (sample, label) pair (terms ``class'' and ``label'' are used alternatively). Each instance has a binary status, that is, 1 to 0 or positive to negative. An instance with a positive (negative) status indicates that the corresponding sample has (does not have) the associated label as the ground truth. Specifically, we use $Q_{k, m}$ as a binary indicator for whether the status of instance $(m, k)$ is positive, or whether sample $m$ truly has label $k$.

Assume there are $M$ samples and $K$ candidate labels organized in tree structures. In graph language, each of the $M$ samples is represented by a forest consisting of $K$ nodes, that is, each node is an instance as shown in Figure \ref{fig:notation_graph}. Denote by $\mS$ the set of all candidate nodes of the $M$ samples. Under the tree constraint, $\mS$ is a forest of size $n = KM$. Our purpose is to sort these nodes/instances to maximize some objective quantity (see Section \ref{sec:eAUC}) with the hierarchy constraint satisfied. We denote the list of the sorted nodes by $\mL$.

We define a chain $\mathcal{C}_{r \rightarrow e}$ as a directed path that starts from node $r$ and ends at node $e$ (the path is unique for a tree). For the sake of simplicity, we use $\mathcal{C}_r$ to represent $\mathcal{C}_{r \rightarrow e}$ if node $r$ has at most one descendant in each generation and node $e$ is a leaf node. We call $\mathcal{C}_r(i)$ a sub-chain that consists of the first $i$ nodes of $\mathcal{C}_r$. As illustrated in Sample 1 of Figure \ref{fig:notation_graph}, $\mathcal{C}_{\text{A}_1 \rightarrow \text{E}_1}$ represents the chain $\text{A}_1 \rightarrow \text{B}_1 \rightarrow \text{D}_1 \rightarrow \text{E}_1$, $\mathcal{C}_{\text{B}_1}$ represents the chain $\text{B}_1 \rightarrow \text{D}_1 \rightarrow \text{E}_1$, and $\mathcal{C}_{\text{B}_1}(2)$ is a sub-chain of $\mathcal{C}_{\text{B}_1}$ consisting of $\mathcal{C}_{\text{B}_1}$'s first two nodes, i.e., $\text{B}_1 \rightarrow \text{D}_1$. For a node $k$, we denote its LPR value as $LPR_k$. Given a chain $\mathcal{C}_r$, we define the average LPR of sub-chain $\mathcal{C}_r(i)$ as $\bar{\ell}_{r, i} = \frac{1}{|\mathcal{C}_r(i)|}\sum_{k \in \mathcal{C}_r(i)} LPR_k$. For a forest $\mS$, we denote
\begin{equation}\label{eq:def_P}
  \mathcal{P}:=
  \left \{
  \begin{array}{ll}
     v: & v~ \text{has multiple child nodes, while all of}\\
     & \text{the child nodes have at most one}\\
     & \text{descendant in each generation}.
  \end{array}
  \right \}
\end{equation}
 The subtree starting from any child of $v \in \mathcal{P}$ would be a single chain. For example, in Sample 1 of Figure \ref{fig:notation_graph}, only node $\text{F}_1$ belongs to $\mathcal{P}$. Node $\text{B}_1$ and node $\text{H}_1$ do not belong to $\mathcal{P}$ since they only have one child. Node $\text{A}_1$ and node $\text{C}_1$ do not belong to $\mathcal{P}$ because their common descendant node $\text{F}_1$ has two children. 

 \subsection{LPR and Hit Curve}\label{sec:lpr_hit_curve}
 \noindent \textbf{LPR}. Suppose there are pre-trained classifiers for the $K$ classes and the classification scores for the $M$ samples are $\{s_{k, m}|m = 1, \ldots M$, $k = 1\, \ldots, K\}$. A larger score implies a larger chance to be positive in the corresponding class. Let $x$ be a random object ($m$ is a realized sample of $x$) and $S_{k,x}$ be a classifier score (with cdf $F_k$) for $x$ against class $k$. Let $\lambda_{k, u_k}$ be a cutoff for scores from the classifier of class $k$; $u_k$ represents the desired chance that $x$ is not assigned to class $k$, or equivalently, $u_k = P(S_{k, x} \leq \lambda_{k, u_k}) = F_k(\lambda_{k, u_k})$. We define the LPR function for class k as 
\begin{equation}
\label{LPR_def}
LPR_k(u_k) = -\frac{d}{du_k}\{(1 - u_k)G_k(u_k)\}, 
\end{equation}
where $G_k(u_k) = P(Q_{k,x} = 1 | S_{k,x} > \lambda_{k, u_k})$. It is shown in \cite{jiang2014} that LPR can be directly compared across classes while the original classification scores are not generally comparable. Sorting LPRs in the decreasing order guarantee the the optimal pooled precision at any pooled recall rate if there is no hierarchy constraint. In addition, LPR was shown to be equivalent to the local true discovery rate, $\ell\text{tdr}$; see Appendix \ref{appendix:LPR_inference} for more details on LPR. \\

%%%%%%%%%%%%%%%%
\noindent \textbf{Hit curve}.
In a hit curve, the x-axis represents the number of discoveries and the y-axis represents the number of true discoveries (i.e., the hit number). Hit curve has been explored in the information retrieval community as a useful alternative to a receiver operating characteristics (ROC) or precision-recall (PR) curve, particularly in situations where the users are more interested in the top-ranked instances. For example, in evaluating the performance of a web search engine, the relevance of the top-ranked pages is more important than those that appear lower in search results because users expect that the most relevant results appear first. The hit curve can serve well in this situation as a graphic representation of the ranker's performance, since it would plot the results in order of decreasing relevance and the y-axis would indicate the results' relevance to the target. On the other hand, the ROC curve, which plots the true positive rates (TPR) against the false positive rates (FPR) at various threshold settings,  does not depend on the prevalence of positive instances\citep{davis2006, hand2009}. In the case of search results, the number of relevant pages is tiny compared to the size of the World Wide Web (i.e., low prevalence of positive instances), which would result in an almost zero FPR for the top ranked pages. That is to say, with very few true positives, the early part of the ROC curve would fail to meaningfully visualize the search ranking performance. In the case of many hierarchical multilabel classification problems, like disease diagnosis problems, this issue exists as well; there are  many candidate diseases to consider while few are actually relevant to the patient. Although the PR curve accounts for prevalence to a degree (i.e., showing the tradeoff between precision and recall for different threshold), \cite{herskovic2007} provided a simple example where the hit curve can be the more informative choice: with only five positive cases out of 1000, the hit curve’s shape clearly highlighted the call order of a method that had called 100 instances before the five true positives, whereas the corresponding PR curve was uninformative (i.e., both the recall and precision rates are zero for the first 100 called instances).

\section{Methods} \label{sec:methods}
In this section, we motivate and derive the novel metric eAUC. Then, we specify the HierLPR ranking algorithm and discuss its theoretical properties.

% The decision rule of \citep{jiang2014} for multilabel local classification without a hierarchy constraint is to maximize the pooled precision rate at any given pooled recall rate, which can be realized by sorting the nodes using LPR values. Its superior performance has been shown on simulated and real data. To extend this idea to the HMC problem, we propose an algorithm called HierLPR that aims to maximize the eAUC, a concept derived form the hit curve, under the hierarchy constraint. Doing so prioritizes nodes at the top of the hierarchy, which are often of greater practical importance, and is helpful for attaining computational feasibility; see Section \ref{sec:eAUC}. HierLPR outputs an sorted list, then users can apply a threshold to it to produce label assignments, e.g. taking the top $k$ to be positive calls or learning a cutoff for a targeted precision rate or F-measure.

\subsection{A novel evaluation metric for HMC: early-call-driven area under curve (eAUC)}\label{sec:eAUC}
Measures like ROC-AUC, precision, recall, and F-measure are widely used in classification problems without hierarchy constraint. However, consensus on standard metrics has not yet been achieved for evaluating methods for the HMC problem, and development of such metrics remains an active research topic today. For example, H-loss \citep{cesa2006incremental, rousu2006kernel, cesa2006hierarchical}, matching loss \citep{nowak2010performance}, hierarchical hamming loss and hierarchical ranking loss \citep{bi2015} were designed to blend the false positive rate (TPR) and the false negative rate (TNR), while considering the hierarchical constraint. However, all of these loss functions depend on the choice of per-node cost and other coefficients, which are usually determined with intuition or domain knowledge rather than with statistical justification. On the other hand, the hierarchical versions of precision, recall, and F-measure introduced in \cite{kiritchenko2005, verspoor2006categorization} do not require determination of the per-node costs, but have complicated forms that make them expensive to compute and difficult to optimize over.

The ``right'' kind of error to look at is dependent on the user's end goal, which explains why a standard evaluation metric has not yet been agreed upon \citep{costa2007}. In this paper, we are interested in settings where accuracy in the initial set of positive calls is desirable, rather than capturing all of the true positives in the dataset. This is motivated by the case of disease diagnosis, where physicians need to have confidence that the general category of disease has been correctly diagnosed, and would tolerate or even expect mistakes at more specific levels since those typically need to be corroborated by expert knowledge \citep{huang2010}. In this section, we define a novel metric towards this end. It is motivated by hit curve that is suitable for situations where 1) the number of true positives is small, 2) top calls are of most interest. Our metric not only inherits the two properties from hit curve, but overcomes the issue that optimizing the area under hit curve is intractable due to the nature of counting.

Given $M$ samples and $K$ candidate labels, suppose the classifier scores of the $n = KM$ instances are sorted via some criterion, say in a descending order, so that $s_{(1)}, \ldots, s_{(n)}$ represents the order in which they are called positive. Let $Q_{(1)}, \ldots, Q_{(n)}$ be the true status of these instances. Note that the x-axis of the hit curve represents the number of calls made, thus the expression for the area under the curve is equivalent to the sum of the number of true positives among the top $k$ calls, for every $k$. We use $I\{Q_{(k)} = 1 | s_1, \ldots, s_n\}$ to indicate that the $k$-th instance is a true positive given the classifier scores. This yields the convenient expression for the area under the hit curve:
\begin{eqnarray}\label{eq:AUC_hit}
  &&\text{AUC~of~ hit~curve}\nonumber\\
  &=&\sum_{i=1}^n \sum_{k=1}^i  I\{Q_{(k)} = 1 | s_1, \ldots, s_n\} \nonumber\\
  &=&\sum_{k=1}^n (n-k+1)I\{Q_{(k)} = 1 | s_1, \ldots, s_n\}
\end{eqnarray}
However, \eqref{eq:AUC_hit} is intractable to directly optimize over. To tackle it, we take expected values of \eqref{eq:AUC_hit} and arrive at
\begin{eqnarray}\label{eq:objfun}
  && \mathbb{E} [\text{AUC~of~ hit~curve}]\nonumber\\
  &=&\sum_{k=1}^n (n-k+1) P(Q_{(k)} = 1 | s_1, \ldots, s_n) \nonumber\\
  & \approx &\sum_{k=1}^n (n-k+1) P(Q_{(k)} = 1 | s_{(k)}) \nonumber \\ 
  & = &\sum_{k=1}^n (n-k+1) LPR_{(k)}.
\end{eqnarray}

We use \eqref{eq:objfun}, a function of LPR, to approximate the area under the hit curve, and call it early-call-driven area under curve (eAUC). The approximation step assumes that the k-th score contains sufficient information to infer the status of the associated instance, i.e., $Q_{(k)}$. We reason that the approximation is sensible with the following arguments. First, since top- or near top-level nodes in the hierarchy represent general classes, they are more likely to have well trained classifiers than the others. For these nodes, their own scores can provide strong evidence for the inference of $Q_{(k)}$. Second, nodes of strong signal, for example, positive top or near-top nodes or patients with obvious symptoms in reality, should have large scores and rank top in the ordering. Thus, for nodes with small k in the optimal topological ordering, the probability $P(Q_{(k)} = 1| s_1, s_2, \ldots, s_n)$ and its weight $(n - k + 1)$ are large, which indicates that the terms with small k dominate over those with large k. Therefore, \eqref{eq:objfun} is an appropriate approximation for the expected area of the hit curve. It not only gives rise to the computational feasibility, but also emphasizes our target on the top nodes.

\subsection{A novel ranking algorithm for HMC: HierLPR}\label{sec:method_HierLPR}
\begin{algorithm}[H]
\caption{The Chain-Merge algorithm.}\label{algo:chain_merge}
\textbf{Input: }$p$ chains $\mS = \{\text{node} \in \mathcal{C}_r: r = r_1, \ldots, r_p\}$.\\
\textbf{Procedure: }
\begin{algorithmic}[1]
\STATE Set $\mL=\emptyset$;
\STATE Compute $\{\bar\ell_{r, i}: i = 1, \dots|\mathcal{C}_r|, r = r_1, \ldots, r_p\}$.
\WHILE{$\mS \neq \emptyset$}
\STATE $(r',i') = \arg \underset{C_r(i) \subset \mS}{\max} \bar\ell_{r,i}.$
\STATE $\mL \gets \mL \oplus C_{r'}(i')$, where $\oplus$ indicates the concatenation of two sequences.\\
\STATE $\mS \gets (\mS \backslash C_{r'}) \cup (C_{r'}\backslash C_{r'}(i'))$.\\
\STATE Update the average LPR's of the remaining nodes. 
\ENDWHILE      
\end{algorithmic}
\textbf{Output: }$\mL$.
\end{algorithm}

Now we reduce the HMC problem to maximizing eAUC under the hierarchy constraint, for which the solution is a topological ordering of instances. Throughout this section, we only consider the HMC problem with the forest hierarchy -- the labels follow a tree structure and the entire graph may consist of several trees (Figure \ref{fig:notation_graph}). The ordering task can be thought of as merging all of the forests into a single chain. Starting from a simple case, Algorithm \ref{algo:chain_merge} provides a procedure for merging multiple chains, which is illustrated by Figure \ref{fig:merge_illustration} (i). Algorithm \ref{algo:chain_merge} acts as the building block for Algorithm \ref{algo:original_algo} to merge multiple complicated trees with respect to the hierarchical structure, which is illustrated by Figure \ref{fig:merge_illustration} (ii).

\begin{algorithm}[H]
\caption{The HierLPR algorithm.}\label{algo:original_algo}
\textbf{Input: }{A forest $\mS$}\\
\textbf{Procedure: }
\begin{algorithmic}[1]
\STATE Figure out $\mathcal{P}$ (defined in \eqref{eq:def_P}).
\WHILE{$\mathcal{P} \neq \emptyset$}
\STATE  Pop out one $v$ from $\mathcal{P}$. Take two children of $v$, $r_1$ and $r_2$.
\STATE  Feed $C_{r_1}$ and $C_{r_2}$ into Algorithm \ref{algo:chain_merge} and obtain $\mL(r_1, r_2)$.
\STATE  Replace $C_{r_1}$ and $C_{r_2}$ with $\mL(r_1, r_2)$.
\STATE  Update $\mathcal{P}$.
\ENDWHILE
\IF{There remain multiple chains}
\STATE Apply Algorithm \ref{algo:chain_merge} to these chains.
\ENDIF
\STATE Let $\mL$ be the resulting chain.
\end{algorithmic}[1]
\textbf{Output: }{a sorted list $\mL$}.
\end{algorithm}

\begin{figure}
  \centering
   \includegraphics[width=0.9\textwidth]{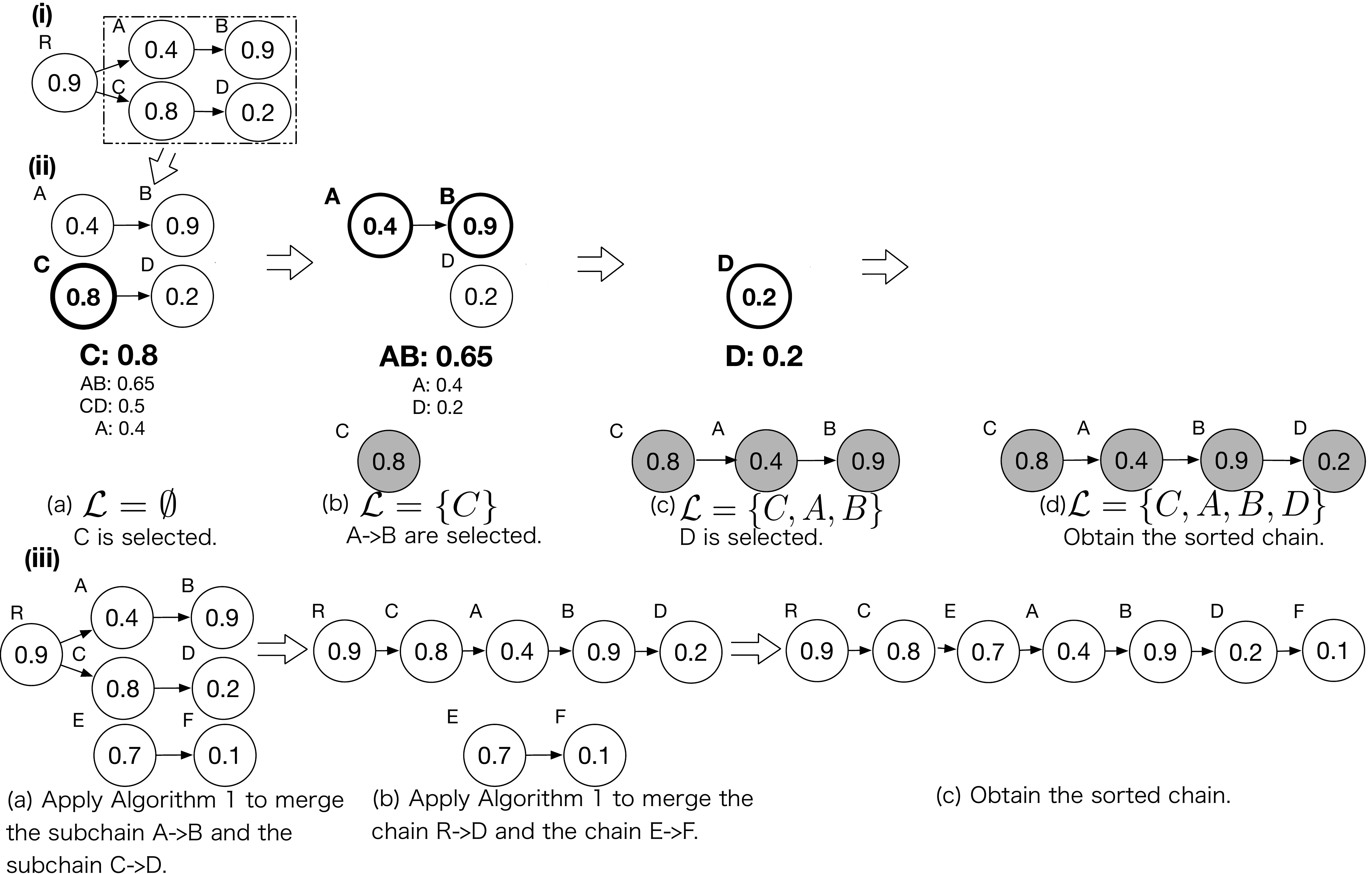}
   \caption{An example of the merging process in Algorithm \ref{algo:chain_merge} and Algorithm \ref{algo:original_algo}. (i) A tree of five nodes to be merged. (ii) Merge the two sub-chains of two nodes in the dash square in (i) using Algorithm \ref{algo:chain_merge}. The nodes in bold form a sub-chain of the highest averaging LPR, and the nodes filled in grey gives a sorted list produced by the merging procedure; (iii) Merge two trees using Algorithm \ref{algo:original_algo}.}
     \label{fig:merge_illustration}
\end{figure}

Algorithm \ref{algo:original_algo} is inspired by the idea that sorting the nodes from below does not affect the sorting of the nodes above in the hierarchical structure. Therefore, it is natural to merge the nodes from bottom to top. Another fundamental nature of Algorithm \ref{algo:original_algo}, originating from Algorithm \ref{algo:chain_merge}, is to detect and remove the largest moving average along the chain. It makes the strong postive nodes from lower down in the hierarchy to help weak signals at the top appear early in the ranking list. Thus, HierLPR dos not suffer from the blocking problem that is common in local classifiers. Finally, Algorithm \ref{algo:original_algo} can be justified from a theoretical perspective ---  Theorem \ref{thm:opt_original_algo} shows that HierLPR gives rise to a topological ordering that maximizes eAUC. 

Nonetheless, Algorithm \ref{algo:original_algo} is inefficient due to the exhaustive merging and repeated computations of the moving average at each iteration. For each sample, the time complexity reaches up to $\mathcal{O}(K^3)$, and thus merging all the $M$ samples costs $\mathcal{O}(nK^2 + n \log M)$ computations. By eliminating the redundant computations, we can provide a faster implementation of Algorithm \ref{algo:original_algo} that costs $\mathcal{O}(n \log n)$ time complexity (See Section \ref{sec:faster_version}). On the other hand, we note this issue is not a practical concern because: (i) the hierarchy structure is given and fixed, so $K$ can be regarded as a constant. In the scenarios of our interest like the disease diagnosis, $K \sim 100$ in most cases; (ii) parallel computing is more helpful than improving the worst-case constant in terms of computational time.

\begin{theorem}\label{thm:opt_original_algo}
Finding out the optimal (in terms of eAUC) topological ordering of the original structure is equivalent to finding out the optimal topological ordering of the structure that replaces any sub-tree with the corresponding merged chain by Algorithm \ref{algo:original_algo}. Hence, Algorithm \ref{algo:original_algo} leads to the optimal topological ordering when all the trees are merged into a single chain.
\end{theorem}
\begin{proof}
  The simplest case is a graph with only one chain, and the theorem obviously holds for this case. Next, suppose there are two chain branches, $X_{(1)}, \ldots, X_{(m)}$ and $Y_{(1)}, \ldots, Y_{(m^{\prime})}$, both of which share the same parent node in $\mathcal{P}$. Directly merging these two chain branches by Algorithm \ref{algo:chain_merge} yields an ordering denoted as $\mathcal{O}_{XY}$. Now given an arbitrary ordering $\mathcal{O}_{A}$ of all the $n$ nodes in the entire graph, which respects the tree hierarchy. Denote by $p_1$ the first position of the nodes among $X_{(i)}'s$ and $Y_{(j)}'s$ within $\mathcal{O}_A$, and denote all the nodes, other than $X_{(i)}'s$ and $Y_{(j)}'s$ and located after the $p_1$ position, by $W_{(1)}, \ldots, W_{(n^{\prime})}$ (the position of $W_{(l)}$ is ahead of that of $W_{(l^{\prime})}$ in $\mathcal{O}_A$ if $l < l^{\prime}$). We want to show that
\begin{lemma}\label{lemma:equiv_chain}  
There exists an ordering $\mathcal{O}_{A^{\prime}}$ that is at least as good as $\mathcal{O}_A$, in terms of eAUC, such that {\Large $(\star)$}{\em ~Node $B$ is located ahead of Node $C$ in $\mathcal{O}_A$ if it is the case in $\mathcal{O}_{XY}$, where $B$, $C$ are two distinct nodes from $X_{(i)}'s$ and $Y_{(i)}'s$.}
\end{lemma}
Lemma \ref{lemma:equiv_chain} implies that in order to figure out the optimal ordering of the original tree structure, it boils down to replacing the two branches $X_{(1)}, \ldots, X_{(m)}$ and $Y_{(1)}, \ldots, Y_{(m^{\prime})}$ by a single chain characterized by $\mathcal{O}_{XY}$ and seeking the optimal ordering of the new structure. The ordering $\mathcal{O}_{A^{\prime}}$ mentioned above can be constructed easily with two constraints: (i) fixing the nodes located ahead of the position $p_1$ as well as their ordering as in $\mathcal{O}_A$; (ii) the position of $W_{(l)}$ is ahead of $W_{(l^{\prime})}$ if $l < l^{\prime}$, as in $\mathcal{O}_A$. The first constraint is straightforward, and the second constraint can be satisfied by applying Algorithm \ref{algo:chain_merge} to $X_{(1)}, \ldots, X_{(m)}$, $Y_{(1)}, \ldots, Y_{m^{\prime}}$ and $W_{(1)}, \ldots, W_{(n^{\prime})}$. Here, we take $W_{(1)}, \ldots, W_{(n^{\prime})}$ as a chain, regardless of their original structure. Without loss of generality, we assume the first maximal chain branch figured out by Algorithm \ref{algo:chain_merge} is $X_{(1)}, \ldots, X_{(t)}$, $t \leq m$ (we can skip the case that a part of $W_{(l)}$'s is the first maximal chain branch since it does not affect {\Large ($\star$)}). In order to prove Lemma \ref{lemma:equiv_chain}, it is reduced to showing that

\begin{lemma}\label{lemma:first_subchain}
{\em Conditional on (i) and (ii), the ordering of the maximal eAUC puts $X_{(1)}, \ldots, X_{(t)}$ at the position $p_1, \ldots, p_1 + t -1$ respectively.}
\end{lemma}
The detailed proof of Lemma \ref{lemma:first_subchain} is deferred to Appendix \ref{appendix:proof_lemma_first_subchain}. Note that $X_{(1)}, \ldots, X_{(t)}$ must be located in the first place in the ordering $\mathcal{O}_{XY}$. Therefore, by applying Lemma \ref{lemma:first_subchain} in an inductive way (exclude $X_{(1)}, \ldots, X_{(t)}$ and apply the same argument on the remaining nodes), we can conclude the constructed $\mathcal{O}_{A^{\prime}}$ is at least as good as $\mathcal{O}_{A}$ and satisfies {\Large $(\star)$}. Here, we need to clarify the point that putting $X_{(1)}, \ldots, X_{(t)}$ in such place does not violate the tree hierarchy since $X_{(i)}'s$ and $Y_{(j)}'s$ are the children chains of the same node, and none of $W_{(l)}'s$ can be an ancestor of $X_{(i)}'s$ or $Y_{(j)}'s$, otherwise $\mathcal{O}_A$ is not a valid ordering. Thus the proof is completed.
\end{proof}

\subsection{Faster versions of HierLPR}\label{sec:faster_version}
To solve the scalability issue of HierLPR, we propose a faster version of HierLPR by reducing redundant and repetitive computations in Algorithm \ref{algo:original_algo}. The speed-up is motivated by the following observations: 1) Algorithm \ref{algo:chain_merge} breaks a single chain into multiple blocks via the formula $(r',i') = \arg \underset{C_r(i) \subset \mS}{\max} \bar\ell_{r,i}.$; 2) It can be shown that these blocks can only be agglomerated into a larger block rather than being further partitioned into smaller ones (see Appendix \ref{appendix:faster_algo}); 3) the agglomeration occurs only between a parent block and its child block in the hierarchy. Thus, HierLPR can be implemented at the block level so that the partition is only executed once during multiple merging. By considering the above facts and taking care of other details, we obtain a faster version of HierLPR (see Algorithm \ref{algo:faster_algo}), which  costs $\mathcal{O}(n \log n)$ computations. Its implementation and other details are delegated to Appendix \ref{appendix:faster_algo}.

Alternatively, we found an existing algorithm called Condensing Sort and Select Algorithm (CSSA) \citep{baraniuk1994} that is also of $\mathcal{O}(n\log n)$ complexity and can be adapted to solve \eqref{eq:objfun}. \cite{bi2011} first extended CSSA in their proposed decision rule for the HMC problem. In their paper, CSSA was used to provide an approximate solution to the integer programming problem
\begin{eqnarray}
\max_{\Psi}&& \sum_{k \in \mathcal{T}} B(k) \Psi(k)\label{opt:bi2011_original}\\
s.t.&& \Psi(k) \in \{0, 1\},\forall k, \quad \sum_{k \in \mathcal{T}} \Psi(k) = L,  \label{const:binary}\\
&&\Psi \text{ is }\mathcal{T}\text{-nonincreasing},\nonumber
\end{eqnarray}
where $B(k)$ is a score produced by kernel dependency estimation (KDE) approach \citep{weston2003kernel}. Instead of directly solving \eqref{opt:bi2011_original} with \eqref{const:binary}, \cite{bi2011} tackles a relaxed problem by replacing the binary constraint \eqref{const:binary} by
\begin{equation}
  \label{const:relaxed}
  \Psi(k) \geq 0, \forall k, \quad \Psi(0) = 1, \sum_{k \in \mathcal{T}} \Psi(k) \leq L.
\end{equation}
It turns out that CSSA can be modified as Algorithm \ref{algo:equiv_algo} that is shown to generate the same result as Algorithm \ref{algo:original_algo} (see Theorem \ref{thm:equiv_algo}). On the other hand, we note CSSA and Algorithm \ref{algo:original_algo} differ in the following aspects. First, Algorithm \ref{algo:original_algo} is independently introduced and interpreted in the context of eAUC, with a statistical justification for ordering nodes using LPR in particular. CSSA originates in signal processing and has been successfully used in wavelet approximation and model-based compressed sensing \citep{baraniuk1994, baraniuk1999optimal, baraniuk2010model}. Second, %% the original CSSA cannot theoretically guarantee maximizing eAUC, although the selected nodes by CSSA are almost the same as the first $L$ nodes given as the output of Algorithm \ref{algo:equiv_algo}.
the optimality results of \cite{bi2011} cannot be directly applied to derive the optimality of Algorithm \ref{algo:equiv_algo}, because the relaxed constraint \eqref{const:relaxed} allows $\Psi(k)$'s that are assigned to the nodes in the last iteration to be fraction values rather than $1$ (see detailed discussion in Appendix \ref{appendix:CSSA_discussion}). In addition, Algorithm \ref{algo:original_algo} merges the chains from the bottom up, rather than as in CSSA, constructing ordered sets of nodes called supernodes by starting from the node with the largest value in the graph and moving outward. It is easy to see that the blocks defined in Algorithm \ref{algo:faster_algo} (the faster version of HierLPR mentioned above) are exactly the same as the supernodes taken off in Algorithm \ref{algo:equiv_algo}. Hence, our independently proposed algorithm provides some novel insight into CSSA under the HMC setting.

\begin{theorem}\label{thm:equiv_algo}
Algorithm \ref{algo:original_algo} and Algorithm \ref{algo:equiv_algo} yield the same ordering, so Algorithm \ref{algo:equiv_algo} maximizes eAUC as well.
\end{theorem}

\begin{algorithm}[H]
  \caption{An equivalent algorithm modified from CSSA.}\label{algo:equiv_algo}
  \textbf{Input: }A forest $\mathcal{S}$\\
  Denote $Par(S_i)$ as the parent of supernode $S_i$,  $n(S_i)$ as the number of nodes in $S_i$, and $\mL$ as a vector for holding sorted LPR values.\\
\textbf{Procedure: }
\begin{algorithmic}[1]  
\STATE Initialize with one node per LPR value, and each node as its own supernode, $\mL = []$ (empty vector).
\WHILE{$|\mL| < n$}
\STATE Find $i = \argmax_i\ \frac{1}{n(S_i)} \sum_{j \in S_i} LPR_j$
\IF{$Par(S_i) = \emptyset$}
\STATE Take the nodes in $S_i$ off the graph and append them to $\mL$.
\ELSE
\STATE Condense $S_i$ and $Par(S_i)$ into a supernode.
\ENDIF
\ENDWHILE
\end{algorithmic}
\textbf{Output: }A hierarchically consistent ordering of $n$ LPR values.
\end{algorithm}

\section{Experiments} \label{sec:experiments}
Given the first-stage classification scores generated by arbitrary classifiers, the LPR values are computed by fitting a local quadratic kernel smoother as \cite{jiang2014}. To use ClusHMC, we followed the implementation of \cite{lee2013}. More details on this implementation can be found in Appendix \ref{appendix:clusHMC_details}.

\subsection{Synthetic Data}\label{sec:simulation}
We examined the performance of HierLPR against ClusHMC on four simulation settings shown in Figure \ref{class_sim_settings}. The settings are comprised of simple three-node structures with mixes of high- and low-quality nodes and varying levels of dependence between the nodes. Here, the quality of a node refers to the ability of the corresponding classifier to distinguish between the positive and negative instances. These toy settings are not meant to emulate a complex real world hierarchy, but are designed to be small and simple so that the factors affecting a classifier's performance could be studied in a controlled way where their effects could be better isolated. Settings 1 and 2 are designed to investigate the influence of the high-quality nodes sitting at the top of the hierarchy, and they differ in the positions of the low-quality nodes. Setting 3 targets at low-quality nodes placed at the top of the hierarchy. We study Setting 4 to understand how each method performs when there are standalone nodes. In addition to the above studies, a comparison on a large graph with more realistic and complicated hierarchical relationships follows using real data in Section \ref{sec:real_data}.

\begin{table}[ht]
  \caption{Score distribution in terms of the node quality.} 
\label{tbl:quality_nodes}
  \centering
  \begin{tabular}{l|lll}
    \hline
    Quality & Positive instance & Negative instance & Node color\\
    \hline
    High    & Beta$(6, 2)$      &  Beta$(2, 6)$     & light grey\\
    Medium  & Beta$(6, 4)$      &  Beta$(4, 6)$     & medium grey\\
    Low     & Beta$(6, 5.5)$    &  Beta$(5.5, 6)$   & black\\
    \hline
\end{tabular}
\end{table}

\begin{figure}[h!]
  \centering
  \caption{Three-node graphs tested under simulation. Light grey, medium grey, and black correspond to high, medium, and low quality, respectively.}
\label{class_sim_settings}
\includegraphics[width=0.6\textwidth]{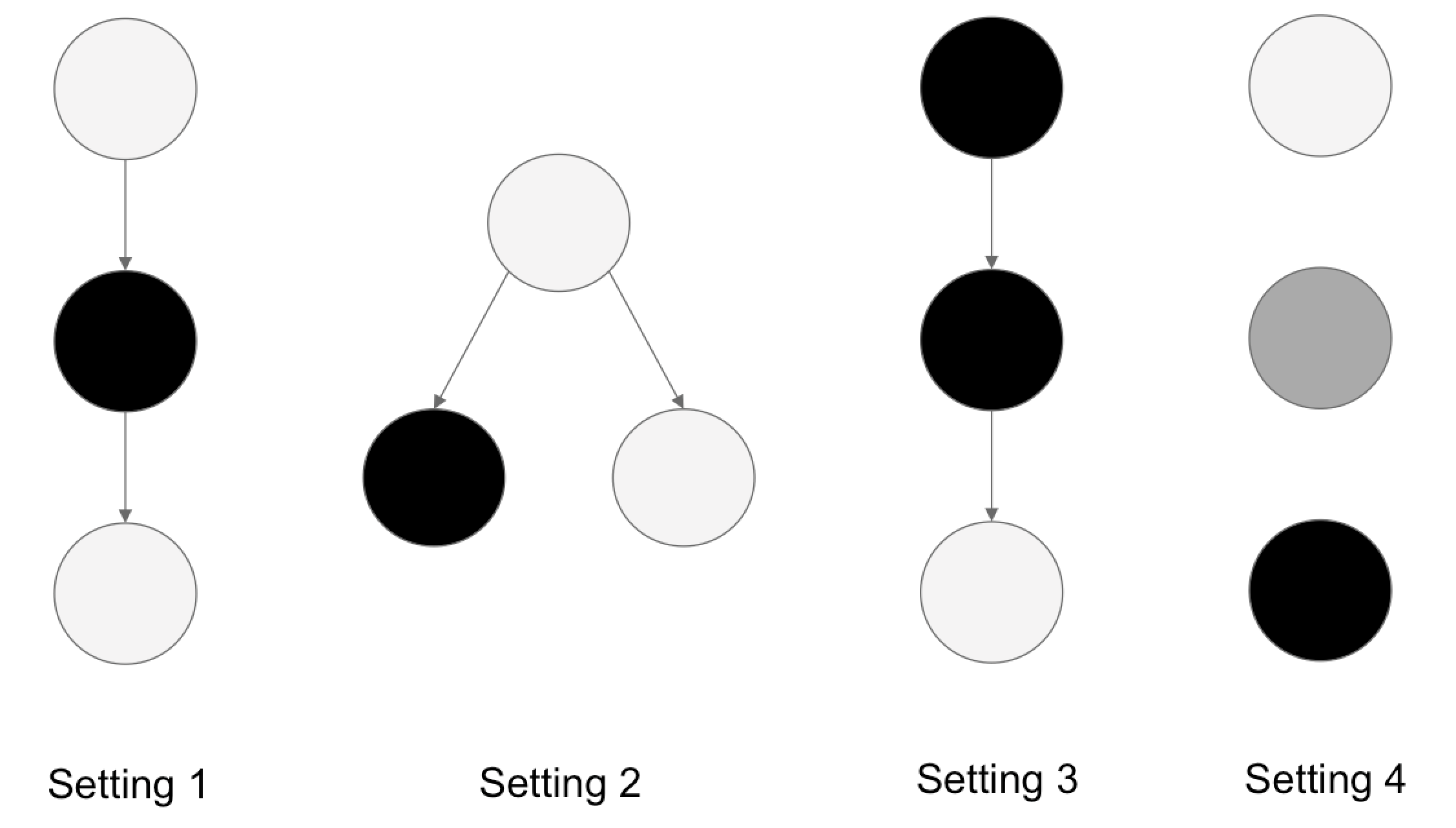}
% \begin{minipage}{0.7\textwidth}
%   \includegraphics[width=4.in]{tree25_setting.png}
% \end{minipage}
\end{figure}

For each simulation setting, 100 datasets are generated. Each simulation dataset consists of 5,000 training instances and 1,000 test instances. We generate the true instance status as follows. For each dataset from a hierarchical setting, the conditional probabilities $P(Q_i = 1 | Q_{\text{parent}(i)} = 1)$ are randomly generated from a uniform distribution, with the constraint that each dataset has to have a minimum of 15 positive instances in the training set, which amounts to a minimum prevalence of 0.3\% for any class. Given the instance status, the instance score (the first-stage classification score) is sampled from the status-specific distribution --- data are generated from a Beta($\eta$, 6) distribution for the negative case and a Beta(6, $\eta$) distribution for the positive case, where $\eta = 2,~4,~5.5$ for the high, medium, low node quality respectively. Details of the score generation mechanism can be found in Table \ref{tbl:quality_nodes}.

The snapshot precision-recall curves with recall rates less than $0.5$, averaged over all 100 replications, for methods ClusHMC, HierLPR, and LPR (sorting LPRs regardless of the hierarchy constraint) are displayed in Figure \ref{class_sim_PR1}. As expected, HierLPR performs better than ClusHMC under settings 1 and 2, because HierLPR is designed to concentrate on the nodes at top levels. Setting 4 is favorable to HierLPR in that it is reduced to the method of \cite{jiang2014} that maximizes the pooled precision rate at any pooled recall rate. In contrast, ClusHMC determines its splits assuming dependence, so when the nodes are independent as in the non-hierarchical case, any rules or association it finds are due to random noise. On the other hand, the comparative performance of HierLPR worsens as the top nodes are of low quality in setting 3. This result is caused by the true positives seated at the low-quality top nodes that are associated with low LPR values. The LPR-alone approach looks better in the beginning because the trouble caused by low-quality nodes can be ignored by violating the hierarchy. However, we see a sharp drop in its precision-recall curve after the initial calls, and subsequently it is dominated by the curve of HierLPR. This observation validates the fact that the hierarchy constraint can help correct the prediction of noisy ancestors if there are children nodes of high quality, as shown in setting 3 in Figure \ref{class_sim_PR1}, or settings 1 and 3 in Figure \ref{class_sim_PR2} in Appendix \ref{appendix:more_results}. However, we notice that HierLPR is more affected by the low-quality top nodes than ClusHMC. This result can be intuitively understood from the fact that HierLPR pays most attention to the top level nodes, although to some extent it allows information exchange among the nodes as ClusHMC or other global classifiers. 

% \begin{figure}[h!]
%   \caption{Precision recall curves comparing ClusHMC, HierLPR, and LPR for settings 1 through 4. The recall rates are limited between 0 and $0.5$ since we mainly focus on the initial calls.}
%     \label{class_sim_PR1}
%   \begin{minipage}{1.0\textwidth}
%     \includegraphics[width=0.51\textwidth]{sim_main1_aug.pdf}
% %    \subcaption{Favorable cases.}
%   \end{minipage}\\
%   \begin{minipage}{1.0\textwidth}
%     \includegraphics[width=0.51\textwidth]{sim_main2_aug.pdf}
% %    \subcaption{Unfavorable cases.}
%   \end{minipage}
% \end{figure}

\begin{figure}[h!]
  \centering
  \includegraphics[width=0.8\textwidth]{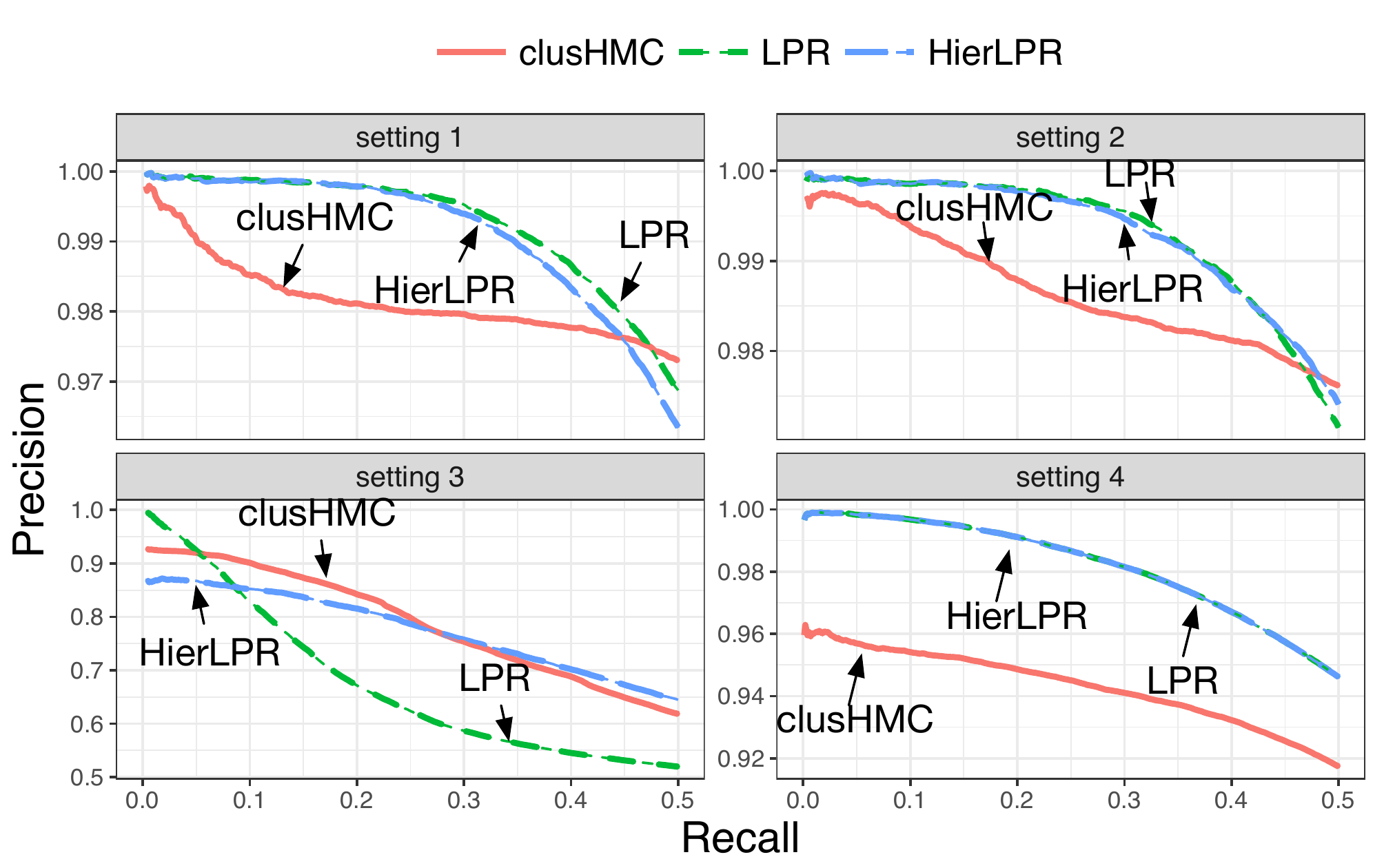}
  \caption{Precision recall curves comparing ClusHMC, HierLPR, and LPR for settings 1 through 4. The recall rates are limited between 0 and $0.5$ since we mainly focus on the initial calls.}
  \label{class_sim_PR1}
\end{figure}

Altogether, ClusHMC is advantageous when considering a connected graph with abnormal distributions of node quality (e.g. top nodes are of low quality), since it takes into account the hierarchical structure at large and allows thorough information exchange among the nodes. This advantage is generally shared by global classifiers. As a local classifier approach, it is not surprising that HierLPR would not prevail over ClusHMC in this regard. HierLPR performs best on classification problems where graphs are either shallow with high-quality classifiers sitting at the top of hierarchy or standalone nodes. In practice, the $K$ labels are often organized into several hierarchies, varying in their depth and breadth (see the disease diagnosis example in section \ref{sec:real_data}). In such a situation, HierLPR would work better since it has an explicit global objective function to guide the rankings across the disjoint hierarchies. By contrast, ClusHMC does not have an optimized strategy to handle disjoint hierarchies. In addition, HierLPR is highly flexible like other local classifiers --- it can be easily expanded to accommodate the addition of new labels or the deletion of unwanted ones. In contrast, ClusHMC has to retrain the datasets from scratch when the graph is altered. 

\iffalse
\begin{table}[h!]
\centering
\caption{The average area under the hit curve over 100 replications for ClusHMC, LPR, and HierLPR under each simulation setting tested.}
\label{hit_area_table}
\begin{tabular}{c|lll}\hline
        & \multicolumn{3}{c}{Area Under Hit Curve ($\times 10^6$)}        \\ 
Setting & ClusHMC               & LPR              & HierLPR \\ \hline
       1 & 188.366 (99.924) & 187.604 (99.599) & 188.044 (99.763) \\
       2 & 188.392 (99.900) & 187.625 (99.621) & 188.023 (99.752) \\
       3 & 88.537 (45.844) & 86.231 (44.886) & 87.852 (45.777) \\
       4 & 185.612 (98.526) & 179.760 (96.330) & 184.712 (98.480) \\
       5 & 186.837 (98.585) & 182.248 (96.521) & 185.234 (98.406) \\
       6 & 165.624 (99.289) &  164.439 (94.263) & 170.938 (95.887) \\
       7 & 100.709 (51.887) & 98.290 (51.083) & 99.446 (51.701) \\
       8 & 11161.03 (3512.758) & 10845.04 (3431.810) & 10953.09 (3458.724) \\
       9 & 11445.38 (3326.868) & 11420.26 (3328.185) & 11445.82 (3334.034) \\
      10 & 323.374 (63.315) & 324.770 (63.301) & 324.770 (63.301) \\
      11 & 303.221 (64.353) & 308.272 (64.155) & 308.272 (64.155) \\ \hline
\end{tabular}
\end{table}
\fi

\subsection{Application to disease diagnosis}\label{sec:real_data}
\cite{huang2010} developed a classifier for predicting disease along the UMLS directed acyclic graph, trained on public microarray datasets from the National Center for Biotechnology Information (NCBI) Gene Expression Omnibus (GEO). They collected 100 studies, including a total of 196 datasets and 110 disease labels. The 110 nodes are grouped into 24 connected tree-like DAGs, that is, these DAGs can be transformed into trees by removing some edges (at most $7$ edges in this case); see Figure \ref{disease_diag_graph1} and \ref{disease_diag_graph2} (Appendix \ref{appendix:char_real_data}) for the detailed graph. In general, the graph has three properties: 1) It is shallow rather than deep; 2) It is scattered rather than highly connected; 3) Data redundancy occurs as an artifact of the label mining: the positive instances for a label are exactly the same as for its ancestors.

We followed the training procedure used in \cite{huang2010} to obtain first-stage local classifiers (see Appendix \ref{appendix:real_data_background}). Given the shallow and tree-like structure, Algorithm \ref{algo:original_algo} can still be applied at an acceptable cost --- split the DAG to all possible trees by assigning each node to one of its parents, apply Algorithm \ref{algo:original_algo} on each possible graph, then choose the one with the highest eAUC. There are at most $7$ nodes with multiple parents in one tree, so at most $2^7 = 128$ possible graphs. It is computationally acceptable in this case since there are only $110$ nodes for one sample. 

We compared the performance of HierLPR against ClusHMC, the first- and second-stage classifier calls of \cite{huang2010}, and the LPR multi-label method of \cite{jiang2014}. The resulting precision-recall curve is shown in Figure \ref{disease_diag_PR}. 

\begin{figure}[ht!]
\centering
\includegraphics[width=0.8\textwidth]{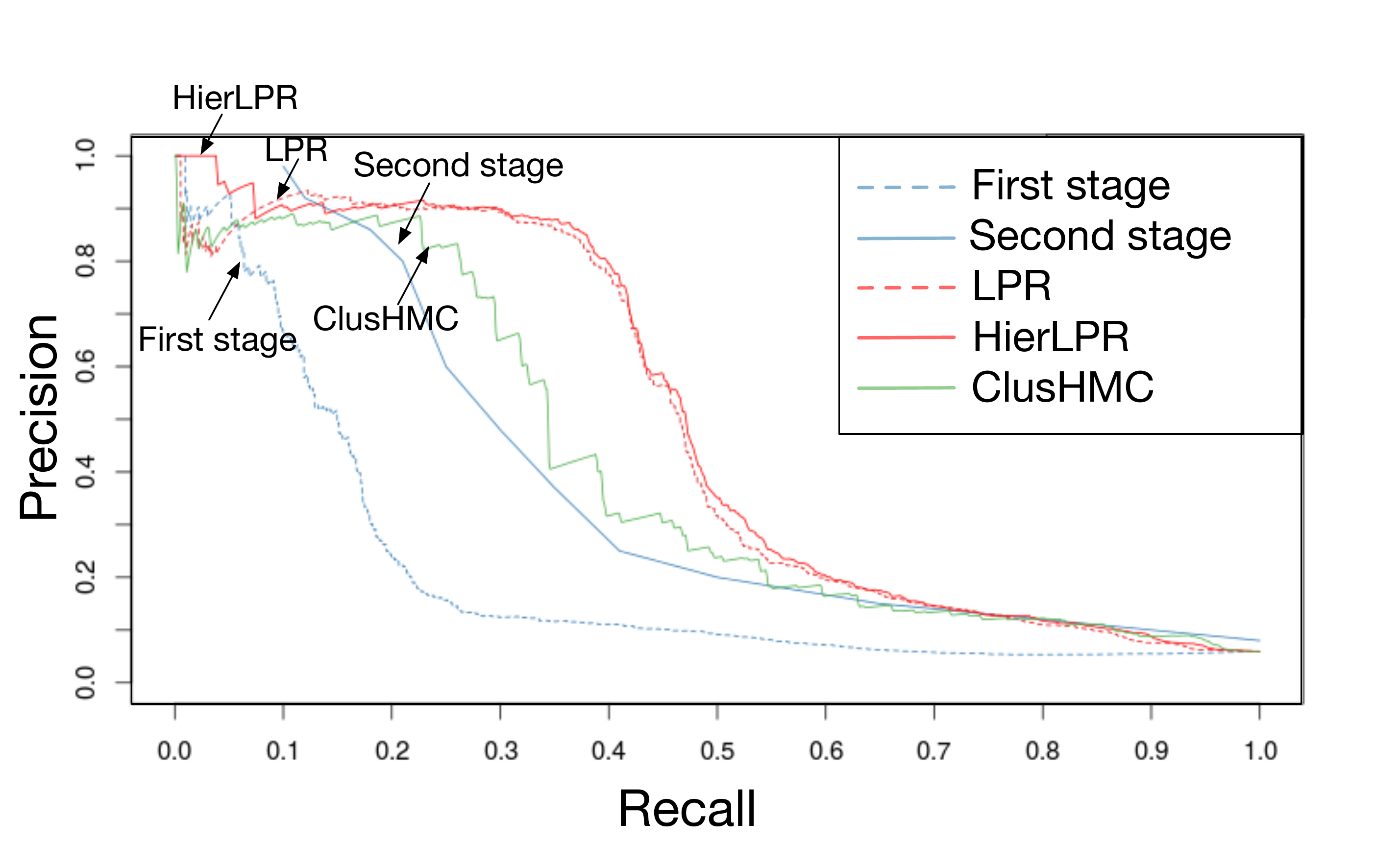}
\caption{Precision recall curve for several classifiers run on the real dataset of \cite{huang2010}.}
\label{disease_diag_PR}
\end{figure}

HierLPR performs better than all of the other methods overall, and it performs significantly better in the initial portion of the precision-recall curve, as we expected from both our theoretical and simulation results. To some extent, the exceptional performance of HierLPR is blessed by the characteristics of the real data  --- over 20\% of the nodes are standalone, while the rest are in shallow hierarchical graphs. Specifically, HierLPR works better in this scenario since it has an explicit global objective function to guide the rankings across the disjoint hierarchies. By contrast, ClusHMC does not have an optimized strategy to handle disjoint hierarchies. %%In practice, the $K$ labels are often organized in a similar styly --- there are several separate hierarchies, varying in their depth and breadth. %%Furthermore, the disease diagnosis data contains a mix of nodes of different qualities.  As we saw in the simulation results earlier, the performance of ClusHMC on the non-hierarchical case is affected by the quality of the nodes.
A closer inspection into the call order of the true positives from the standalone nodes is revealing: for HierLPR, these are some of the first instances to be called; in contrast, ClusHMC calls most of them after over half of the other truly positive instances have been called, illustrating the difficulty ClusHMC has in detecting these particular cases. %In addition, the precision-recall curve of HierLPR eventually overlaps with that of the LPR method: this is because of the shallow graph structure, which makes the problem similar to the non-hierachical multi-label case. Finally, we note that HierLPR is highly flexible like other local classifiers --- it can be easily expanded to accommodate the addition of new labels or the deletion of unwanted ones. In comparison, ClusHMC has to retrain the datasets from scratch when the graph is altered. 

\subsection{Cutoff Selection}\label{sec:cutoff_selection}
Our algorithm can generate an optimal topological ordering in terms of eAUC, but it does not suffice for most practical uses -- people who make decisions need to know where to cut the ranking list to accept the top ones. To this end, we propose splitting the training data into two parts, one for training the individual classifiers and computing the LPRs, and the other for determining the cutoffs. Specifically, the classifiers and LPR computations learned from the first set can be applied to the data of the second set to produce LPR scores. Then HierLPR takes these LPR scores and produces a sorted list of instances for the second set. Since the truth is known on this second dataset, metrics like F-measure or precision (1 - FDR) can be computed for an arbitrary LPR cutoff, so that a desired cutoff in the HierLPR ranking results can be chosen for controlled FDR or maximal F-measure. We denote this cutoff by $LPR^*$. The evaluation on the goodness of this cutoff selection method is performed on the testing set. %Note that the block-level LPR values descend along the ordered list. Given a new ordered list of blocks, we predict that a node belongs to the positive class if it is in the block whose average LPR value is no smaller than $LPR^*$. %% Specifically, one can find the block-level LPR value, $LPR^*$ (see Section \ref{sec:faster_algo} for the definition of the block), associated with the evaluation metric being used, e.g, the optimal F-measure or the targeted $(1 - \alpha) \times 100 \%$ precision ($\alpha$ FDR).

\begin{table}[ht]
  \centering
  \caption{Differences (in percentage) between the measurements computed using the ground truth and those computed using the LPR cutoff selected from the training data. For example, the entry corresponding to setting 1 and ``$90\%$ precision'' refers to $0.9$ minus the testing precision that is calculated using $LPR^*$ associated with $90\%$ precision on the training set. The numbers in the parenthesis indicate the standard deviation in percentage. } 
  \begin{tabular}{c|cccc}
  \hline  
Setting    &  1	& 2	& 3&	 4 \\
\hline
  max F-measure& 0.57 & 0.49 & 7.51 & 0.24 \\
  & \scriptsize (0.64) & \scriptsize(0.64) & \scriptsize (11.64) & \scriptsize (0.45) \\
  $80\%$ precision & 1.25  & 0.71 & 11.31 & 3.13\\
    &\scriptsize (7.7)  & \scriptsize (5.96) & \scriptsize (20.48) & \scriptsize (14.28)\\
  $90\%$ precision & 0.56 & 0.39 & 17.53 & 8.18 \\
    & \scriptsize (3.59) & \scriptsize (2.94) & \scriptsize (18.39) & \scriptsize (9.57) \\
  $95\%$ precision & 0.22 & 0.22  & 20.07 & 1.29\\
    & \scriptsize (1.77) & \scriptsize (1.61)  & \scriptsize (17.2) & \scriptsize (4.56)\\
  $99\%$ precision & 0.13 & 0.12  & 21.69 & 0.16\\
    & \scriptsize (0.68) & \scriptsize (0.64)  & \scriptsize (16.35) & \scriptsize (0.75)\\
\hline
\end{tabular}
\label{tbl:ordering_cutoff}
\end{table}

We take as an example the four simulation settings in Section \ref{sec:simulation} to demonstrate the power of this strategy. Table \ref{tbl:ordering_cutoff} displays the test set results using the LPR cutoff selected by the same criterion on the training set. We note that under favorable settings for HierLPR like settings 1, 2 and 4, it is possible to control the FDR level as desired. On the other hand, for unfavorable settings like setting 3, controlling FDR might be infeasible since no LPR value is associated with the required FDR in the training set. A reasonable alternative is to take as cutoff the LPR value associated with the maximal F-measure in the training set.%, although the variance of the F-measure on the testing set can be larger. 

\section{Conclusions}\label{sec:hmc_conclusion}
In this article, we present a method for performing hierarchical multi-label classification, which is particularly well suited when labels follow a structure with a mix of broad (rather than deep) hierarchies and standalone nodes (rather than highly connected). Such structure is frequently encountered in the real world like disease diagnosis. Our method is developed with the intent of respecting the hierarchy and maximizing eAUC \eqref{eq:objfun}. It has been theoretically shown to optimize eAUC under the tree constraint and light assumptions on dependency between nodes. %% We note the connection between HierLPR and an existing sorting algorithm, CSSA \citep{baraniuk1994}. HierLPR has an advantage over CSSA in terms of statistical interpretations, although the empirical results of both methods are almost the same.

We provide two faster versions of HierLPR of $O(n \log n)$ complexity, one developed by reducing the redundant computations of the original algorithm and the other extended from CSSA. Furthermore, we provide an approach for selecting a cutoff for the final decision on the ranking list from HierLPR. The simulation study shows that this cutoff selection method can control FDR or optimize F-score. Finally, we study the working conditions of HierLPR using the synthetic data, and show its advantage over the competing method clusHMC on the real dataset (NCBI GEO). For these reasons, we recommend HierLPR as a computationally efficient, statistically driven approach for performing the second-stage decision adjustment in a local classification framework.

Despite the merits of HierLPR mentioned above, there remains large room for improvement. First, HierLPR is designed for the local method which trains classifiers for different tasks in isolation. It can be of great interest to use eAUC as an objective function to train classifiers while taking into account the graph hierarchy as the global method. Furthermore, LPR is in fact 1 - $\ell$fdr (local false positive rate), a metric widely used in hypothesis testing. A potential research direction is the extension of the ideas of eAUC and HierLPR from classification to hypothesis testing along hierarchies. More future directions include extending HierLPR to DAG structures, adjusting the weights eAUC to shift focus from top nodes to others, to name a few.

\newpage
%\appendix
%\appendix
\begin{appendices}
%%%%%%%%%%%%%%%%%%
%\iffalse
\section{The local precision rate} \label{appendix:hmc_notation}
\subsection{Problem setting and notation}
For consistency, we use the same notation as \cite{jiang2014}. Assume that classifiers have been learned for $K$ labels connected in an acyclic graph and that there are $M$ instances to be classified. We impose no requirements on class membership outside of being hierarchically consistent: an instance could belong to none of the classes, and those that do belong to a class are not required to have leaf-level membership.  

We assume that each label's classifier was trained on $\widetilde{M}$ instances and produces a score $s_{k, m}$ ($m = 1, \ldots M$, $k = 1\, \ldots, K$) that can be thresholded to produce label assignments: without loss of generality we take larger scores to indicate the positive class, i.e. all instances with $s_{k, m} > \lambda_k$ are said to have label $k$. For example, if a logistic regression is used for predicting label $k$, a standard choice for $s_{k,m}$ is the estimated posterior probability that instance $m$ belongs to label $k$. 

Our classification framework begins with the scores for each instance, for which we assume the following generative model. If $Q_{k,m}$ is a binary indicator for whether instance $m$ truly has label $k$, $Q_{k,m} = 1$ with probability $\pi_k$. We require that label membership implies membership in all of its ancestors: $P(Q_{Par(k),m} = 1 | Q_{k,m} = 1) = 1$, where $Par(k)$ is the parent of label $k$. Also, we assume conditional independence of labels at the same hierarchical level: if labels $k$ and $j$ share a parent $i$, $Q_{k,m}$ and $Q_{j,m}$ are independent conditional on the parent status $Q_{i,m}$. 

Given a threshold $\lambda_k$, the chance that the instance does not belong to the label $k$ is given by $F_k(\lambda_k) = P(s_{k,m} \le \lambda_k)$, the cumulative distribution function (cdf) for the scores of classifier $k$. This CDF can be expressed as a mixture of the score distributions for the two classes: $F_k = \pi_kF_{1,k} + (1 - \pi_k)F_{0,k}$, where $F_{1,k}$ is the CDF for those having the label $k$, and $F_{0,k}$ is the CDF for those without. Analogously, the respective density functions are denoted by $f_{1,k}$ and $f_{0,k}$, and the mixture density by $f_k$.

\subsection{Definition and optimality result}
\cite{jiang2014} developed the local precision rate with the intention of maximizing precision with respect to recall in the multi-label setting. Specifically, they maximized an expected population version of the micro-averaged precision and recall rate given by \cite{pillai2013}. The micro-averaged precision rate has the form $\frac{\sum_k TP_k}{\sum_k (TP_k + FP_k)}$, where $TP_k$ and $FP_k$ are the number of true and false positives for label $k$, respectively.

We can write the expected pooled precision and recall rate, we first write the expected precision of the classifier for class $k$ with threshold $\lambda_k$ as

\begin{equation}
\label{precision_fxn}
G_k(\lambda_k) = P(Q_{k,.} = 1 | s_{k,.} > \lambda_k) = \frac{\pi_k (1 - F_{1,k}(\lambda_k))}{1 - F_k(\lambda_k)}.
\end{equation}

From rearranging we also have that the joint probability $P(s_{k,.} > s \text{ and } Q_{k,.} = 1)$ is $(1 - F_k(s))G_k(F_k(s))$. 

Then, we can pool decisions across all $K$ labels using the thresholds $\lambda_1, \ldots, \lambda_k$ to obtain the expected pooled precision rate (ppr).

\begin{equation}
ppr = \frac{\sum_k (1 - F_k(\lambda_k))G_k(\lambda_k)}{\sum_k 1 - F_k(\lambda_k)}
\end{equation}

The denominator represents the \textit{a priori} expected number of times a given instance will be assigned to a label if the decision thresholds $\lambda_1, \ldots, \lambda_k$ are used. The pooled recall rate ($prr$) has the same form, except with $\sum_k Q_{k,.}$ as the denominator instead. 

\cite{jiang2014} observed that to maximize the expected pooled precision with respect to pooled recall, it was enough to maximize $\sum_k (1 - F_k(\lambda_k))G_k(\lambda_k)$ while holding $\sum_k 1 - F_k(\lambda_k)$ fixed since $\sum_k Q_{k,.}$ was a constant. The local precision rate (LPR) was then defined as

\begin{equation}
\label{LPR_def}
LPR_k(s) = -\frac{d}{dF_k(s)}\{(1 - F_k(s))G_k(s)\} = G_k(s) - (1 - F_k(s))\frac{dG_k(s)}{dF_k(s)}
\end{equation}

In their main theoretical result Theorem 2.1, they showed that if the LPR for each class is monotonic, then ranking the $KM$ LPRs calculated for each instance/class combination and thresholding the result produces a classification that maximizes the expected pooled precision with respect to a fixed recall rate. The monotonicity requirement is equivalent to having monotonicity in the likelihood of the positive class, and it is satisfied when higher classifier scores correspond to a greater likelihood of being from the positive class--this rules out poorly behaved classifiers, for example a multimodal case where the positive class scores lie in the range $[0, 0.3) \cup (0.7, 1]$, and the negative class scores in $[0.3, 0.7]$. 

\subsection{Connection to local true discovery rate}
\label{appendix:ltdr_connection}
After substituting expressions for the derivatives $\frac{dG_k(s)}{dF_k(s)} = \frac{dG_k(s)}{ds}\frac{ds}{dF_k(s)}$, the $LPR$ can be shown to be equivalent to the local true discovery rate, $l\text{tdr}$. 

\begin{align}
\label{LPR_ltdr_form}
LPR_k(s) &= G_k(s) - (1 - F_k(s))\frac{dG_k(s)}{dF_k(s)} \\
&= G_k(s) - (1 - F_k(s)) \left[\frac{\pi_k f_{1,k}(s)}{(1 - F_k(s))f_k(s)} + \frac{\pi_k (1 - F_{1,k}(s))}{(1 - F_k(s))^2} \right] \\
&= G_k(s) - \frac{\pi_kf_{1,k}(s)}{f_k(s)} - G_k(s) \\
&= \frac{\pi_kf_{1,k}(s)}{f_k(s)} = l\text{tdr}
\end{align}

The local false discovery rate, $l\text{fdr} = 1 - l\text{tdr}$ is its more well known relative; it has been studied extensively for Bayesian large-scale inference. This connection between a statistic used for hypothesis testing and the $LPR$, which was developed for classification, suggests the possibility that methodological developments on the $LPR$ in classification could have meaningful implications for statistical inference. We elaborate on this connection in Section \ref{appendix:LPR_inference}.

\subsection{Methods for estimating LPR} \label{appendix:hmc_estim_methods}
The optimality result in \cite{jiang2014} was derived using true $LPR$ values, which are generally unknown in practice. The authors discussed two methods for estimating the $LPR$. In the first method, estimates for $f_{0,k}$, $f_k$, and $\pi_k$ are plugged in after expressing $LPR_k(s)$ as the local true discovery rate. In the second method, a local quadratic kernel smoother is used to simultaneously estimate $G_k(s)$ and $G_k'(s)$ in the definition of LPR. 

Theoretically, \cite{jiang2014} showed that under certain conditions, the first method converges to the true result faster than the second. However, in simulation studies the second method performed better than the first. The difference is due to the difficulty in estimating the densities $f_{0,k}$ and $f_k$ on real data: any situation which would make kernel density estimation difficult would result in poor estimates of $l\text{tdr}$. For example, if the data are observed densely in one or two short intervals and sparsely elsewhere, the kernel density estimate of $f_k$ would have bumps in the sparse regions, making $l\text{tdr}$ unreliable.  Further, because $f_{0,k}$ and $f_k$ are estimated separately, they have different levels of bias and variance; in particular $f_{0, k}$ has larger variance (since it is only estimated from the negative class cases. In comparison, the functions $G_k(u)$ and $G_k'(u)$ are estimated jointly in the second method and $G_k(u)$ is always densely observed, as its domain is score percentiles rather than the scores themselves. 

For both estimation methods, it is possible to obtain LPR estimates that are negative. Users must adopt a heuristic to handle these cases.

\cite{lee2013} suggested an alternative based on the second method that averages the estimates obtained via a weighted spline fit on bagged samples of the data. The addition of weights and bagging were introduced in order to estimate the precision function more robustly in regions supported by less data.

\subsection{Connection to statistical inference}\label{appendix:LPR_inference}
The key distinction between inference and classification is the presence of training data, which allows users to estimate distributions that are assumed unknown in statistical inference. If we choose to ignore the available class distribution information, one can reframe a two-class classification problem as a hypothesis testing problem where the null corresponds to membership in the negative class. The classifier score could be used as a statistic, although this means that one would need to train the classifier on the available data while pretending that they cannot estimate the class distributions. This approach clearly fails to take full advantage of the available data, but is meant to highlight the connection between these two problems. 

The local precision rate is closely related to another statistic used in Bayesian large-scale inference, the local false discovery rate \citep{jiang2014}. The local false discovery rate was motivated by the insight that in large-scale inference, enough data is available to estimate class distributions with some accuracy. As a result, it is possible to use pointwise statistics based on $f_0(s)/f_1(s)$, which may contain more information than their more popular tail-probability counterparts. Most of the literature on this statistic has come from Bradley Efron, who laid the groundwork theory and provided interesting applications of the local false discovery rate in microarray gene expression experiments in \cite{efron2005, efron2007size, efron2012}. \cite{cai2009simultaneous} proved an optimality result similar to that of \cite{jiang2014} for a multiple inference procedure for grouped hypotheses that uses local false discovery rates: their procedure minimizes the false nondiscovery rate subject to a constraint on the false discovery rate.

Research on hierarchical hypothesis testing is limited but growing. \cite{yekutieli2006} first defined different ways to evaluate FDR when testing hypotheses along a tree, and gave a simple top-down algorithm for controlling these error types in \cite{yekutieli2008}. In that work, the hypotheses at each level of the tree were assumed independent. More recently, \cite{benjamini2014}'s work on selective inference provided an algorithm for testing on hypotheses arranged in families along a two-level tree where the parent and child are permitted to be highly dependent, although in so doing they give up control on the global FDR. Beyond the connection with the local false discovery rate, it remains to be seen whether other concepts from classification with LPRs can also be applied to inference. Most of the literature is concentrated on theoretical results that show that certain testing procedures can effectively bound a measure of Type I error. One possibility is that sorting algorithms with origins in computer science, like the one presented in this work, could have meaningful applications as testing procedures.

\section{Proofs of Theorems} \label{appendix:proof_theorems}
\subsection{Proof of Lemma \ref{lemma:first_subchain}}
\label{appendix:proof_lemma_first_subchain}
\begin{proof}\\
Let $a$ denote the average of these $t$ values. For the sake of simplicity, we further assume $p_1 = 1$ and simply denote by $Z_{(1)}, \ldots, Z_{(n - t)}$ the combination of $X_{(t+1)}, \ldots, X_{(m)}$, $Y_{(j)}'s$ and $W_{(l)}'s$. Let the ordering $\mathcal{O}_A$ be as follows:
  $$
  \begin{array}{ccccccccccc}
    Z_{(1)} &\ldots & Z_{(i_1 - 1)} &X_{(1)}& Z_{(i_1)}& \ldots& Z_{(i_{t}-t)} &X_{(t)}& Z_{(i_{t}-t+1)}& \ldots& Z_{(n - t)}\\
    1&\ldots&i_1 - 1& i_1 &i_1 + 1&\ldots&i_t - 1& i_t &i_t + 1&\ldots&n
  \end{array}
  $$
  where $i_c$ is the position of $X_{(c)}$, $c = 1, \ldots, t$. Note that $i_{c+1} \geq i_c + 1$. Denote by $\mathcal{O}_{A_1^{\prime}}$ the ordering of $(X_{(1)}, \ldots, X_{(t)})+\mathcal{O}_A/(X_{(1)}, \ldots, X_{(t)})$, that is, move $(X_{(1)}, \ldots, X_{(t)})$ to the head of $\mathcal{O}_A$. The difference in the value of the objective function (OF) between $\mathcal{O}_{A^{\prime}_1}$ and $\mathcal{O}_A$ can be written as follows:
  % We claim two propositions w.r.t $X_{(1)}, \ldots, X_{(m)}$ as below, with which the theorem is naturally completed.
  % \begin{proposition}\label{prop:arbitrary_pos_top}
  %   The ordering, with an arbitrary choice of $(i_1, \ldots, i_m)$, respects the tree hierarchy given that $i_{c+1} > i_c$, $c = 1, \ldots, m - 1$.
  % \end{proposition}
  % \begin{proposition}\label{prop:optimal_ordering}
  %    The maximal value of the objective function \eqref{eq:objfun} is attained when $i_c = c,~c = 1, \ldots, m$. We will refer to this ordering with $X_{(1)}, \ldots, X_{(m)}$ at the top as the proposed ordering.
  % \end{proposition}
    
  % \begin{itemize}
  % \item \textit{Proof of Proposition \ref{prop:arbitrary_pos_top}}.\\
  %   Decompose the given ordering into multiple disjoint chain branches which respect the tree hierarchy, with $\mathcal{X} := \{X_{(1)}, \ldots, X_{(m)}\}$ regarded as one branch. It is not necessary that the positions of the nodes within one branch are consecutive, e.g., $\mathcal{X}$. Then for any branch $\mathcal{B}$ other than $\mathcal{X}$, there are only two possible scenarios: (i) the first node of $\mathcal{B}$ is the children of one node in $\mathcal{X}$ since the first node of $\mathcal{X}$ must be the root of a tree; (ii) $\mathcal{B}$ and $\mathcal{X}$ are disconnected. In either scenario, $\mathcal{X}$ can be put ahead of $\mathcal{B}$ without breaking the constraint of tree hierarchy.
  % \item\textit{Proof of Proposition \ref{prop:optimal_ordering}}.\\ 
    
\begin{align*}
  \text{OF of } \mathcal{O}_{A^{\prime}_1}  = & \sum_{i=1}^t (n-i+1) X_{(i)} + \sum_{j=1}^{n-t} (n-t-j+1) Z_{(j)}\\
 \text{OF of } \mathcal{O}_{A} = & (n-i_1+1)X_{(1)} + \ldots + (n-i_t+1)X_{(t)} \\
                      & + \sum_{j=1}^{i_1-1} (n-j+1) Z_{(j)} + \ldots + \sum_{j=i_{t}-t+1}^{n-t} (n-(j+t)+1) Z_{(j)}
\end{align*}
 
\begin{align*}
&\text{OF of } \mathcal{O}_{A^{\prime}_1} - \text{OF of } \mathcal{O}_{A} = \left[ (i_1 - 1)X_{(1)} - \sum_{k=1}^{i_1 -1} Z_{(k)} \right] + \ldots + \left[ (i_t - t) X_{(t)} - \sum_{k=1}^{i_t - t} Z_{(k)}\right] \\ 
& = (i_1 - 1)\left[ X_{(1)} - \frac{1}{i_1 -1} \sum_{k=1}^{i_1 - 1} Z_{(k)} \right] + \ldots + (i_t -t)\left[ X_{(t)} - \frac{1}{i_t -t}\sum_{k=1}^{i_t -t} Z_{(k)} \right] \\
& = (i_1 - 1)\left[ \left(X_{(1)} -a\right) + \left(a - \frac{1}{i_1 -1} \sum_{k=1}^{i_1 - 1} Z_{(k)} \right) \right] + \ldots \\
&\quad\quad\quad\quad + (i_t -t)\left[ \left(X_{(t)}  - a \right) + \left(a - \frac{1}{i_t -t}\sum_{k=1}^{i_t -t} Z_{(k)} \right) \right] \\
& = \left[ (i_1 - 1) \left(X_{(1)} - a \right) + \ldots + (i_t - t) \left(X_{(t)} - a \right) \right] \\
&\quad\quad\quad\quad + \left[ (i_1 - 1)\left(a - \frac{1}{i_1 -1} \sum_{k=1}^{i_1 - 1} Z_{(k)} \right) + (i_t - t)\left(a - \frac{1}{i_t -t}\sum_{k=1}^{i_t -t} Z_{(k)} \right) \right]
\end{align*}

It remains to prove both the first term and the second term on the right side are non-negative:
\begin{itemize}
  \item \textit{The first term}. We can rewrite the sum

\begin{equation*}
(i_1 - 1) \left(X_{(1)} - a \right) + \ldots + (i_t - t) \left(X_{(t)} - a \right)
\end{equation*}

as follows

\begin{equation*}
(i_1 - 1)\sum_{k=1}^t (X_{(k)} - a)\  +\   (i_2 - i_1 - 1)\sum_{k=2}^t \left(X_{(k)} - a\right) + \ldots + (i_t - i_{t-1} - 1) (X_{(t)} - a).
\end{equation*}

The first sum $\sum_{k=1}^t (X_{(k)} - a) = 0$ since $a$ is the average. The other sums being nonnegative follows from the fact that $a$ must be at least as large as the smaller averages in the chain, i.e. $a \ge \frac{1}{c}\sum_{k=1}^c X_{(k)}$ where $1 \le c \le t$. In detail, we know that
\begin{align*}
X_{(c+1)} + \ldots + X_{(t)} &= ta - [X_{(1)} + \ldots + X_{(c)}], \quad 1 \le c \le t-1 \\
& \ge ta - ca = (t-c)a, \quad \text{from the fact above} \\
(X_{(c+1)} - a) + \ldots + (X_{(t)} - a) &\ge 0
\end{align*}

Therefore, each sum
\begin{equation}
  \sum_{k=c}^t \left(X_{(k)} - a\right) \ge 0,~~~c = 1, \ldots, t.
  \label{eq:backwards_average}
\end{equation}
It is clear that the expression 

\begin{equation*}
(i_1 - 1)\sum_{k=1}^t (X_{(k)} - a)\  +\   (i_2 - i_1 - 1)\sum_{k=2}^t \left(X_{(k)} - a\right) + \ldots + (i_t - i_{t-1} - 1) (X_{(t)} - a)
\end{equation*}

is exactly zero only when each $X_{(c)} = a$.  

\item \textit{The second term}. We claim that each term in the expression
  $$\left[ (i_1 - 1)\left(a - \frac{1}{i_1 -1} \sum_{k=1}^{i_1 - 1} Z_{(k)} \right) + (i_t - t)\left(a - \frac{1}{i_t -t}\sum_{k=1}^{i_t -t} Z_{(k)} \right) \right]$$
  must be nonnegative, and equality holds only if there is a tie. To see this, we notice that $\sum_{k = 1}^{i_c - c} Z_{(k)}$ can be separated as three sums: $\sum_{k = t+1}^{t_X} X_{(k)}$, $\sum_{k = 1}^{t_Y} Y_{(k)}$ and $\sum_{k = 1}^{t_W} W_{(k)}$, where $t_X \leq m$, $t_Y \leq m^{\prime}$, $t_W \leq n^{\prime}$ and $c = t_X - t + t_Y + t_W$. In terms of the procedure of Algorithm \ref{algo:chain_merge}, it follows that
  $$\sum_{k = 1}^{t_Y} Y_{(k)} \leq t_Y\cdot a ~~~~~\text{and}~~~~~ \sum_{k = 1}^{t_W} W_{(k)} \leq t_W\cdot a.$$
  For the same reason, $\sum_{k = t+1}^{t_X} X_{(k)} \leq (t_X - t)a$, otherwise we have $\frac{1}{t_X}\sum_{k = 1}^{t_X} X_{(k)} > a$ and it violates the condition that $a$ is the average of the first maximal chain branch. So we have
  $$\frac{1}{i_c - c} \sum_{k = 1}^{i_c -c} Z_{(k)} = \frac{1}{i_c - c} [\sum_{k = t+1}^{t_X} X_{(k)} + \sum_{k = 1}^{t_Y} Y_{(k)} + \sum_{k = 1}^{t_W} W_{(k)}] \leq a.$$
\end{itemize}
\end{proof}

\subsection{Proof of Theorem \ref{thm:equiv_algo}}
\label{appendix:proof_theorem_equiv_algo}
%\label{appendix:proof_lemma_first_S1}
\begin{proof}
    To establish the bridge between Algorithm \ref{algo:original_algo} and Algorithm \ref{algo:equiv_algo}, we start from a simple case, i.e., a tree consisting of multiple chains with the same root (the root is in $\mathcal{P}$). Specifically, denote by $R$ the root and these children chain by $C_s := \{X_1^{(s)}, \ldots, X_{k_s}^{(s)}\}$, $s = 1, \dots, \nu$, where $\nu$ is the number of chains, and $k_s$ is the length of the $s_{th}$ chain. Without loss of generality, suppose $S_1 := C_1(i_1) = \{X_1^{(1)}, \dots, X_{i_1}^{(1)}\}$ is the first supernode that has been condensed to $R$, if $R$ has not been taken off, or the first to be taken off after $R$. We claim that
\begin{lemma}\label{lemma:first_S1}
  When merging $C_1, \ldots, C_\nu$, Algorithm \ref{algo:chain_merge} puts $S_1$ in the first place.
\end{lemma}
To show Lemma \ref{lemma:first_S1}, we only need to show $\frac{1}{i_1}\sum_{k \in C_1(i_1)} LPR_{k} \geq \frac{1}{i}\sum_{k \in C_s(i)}LPR_k$, $\forall~i \in \{1, \ldots, k_s\}, s \in \{1, \ldots, \nu\} $. The detailed proof is deferred to Appendix \ref{appendix:proof_lemma_first_S1}. Inductively, it implies that the ordering given by Algorithm \ref{algo:equiv_algo} on such simple case is the same as Algorithm \ref{algo:original_algo}. Therefore, any complicated structure boils down to the above simple case, since we can inductively merge the sub-chains starting from a root in $\mathcal{P}$ using Algorithm \ref{algo:chain_merge}. This completes the proof showing that the results of Algorithm \ref{algo:original_algo} and Algorithm \ref{algo:equiv_algo} are the same.
\end{proof}

\subsection{Proof of Lemma \ref{lemma:first_S1}}
\label{appendix:proof_lemma_first_S1}
\begin{proof}
We show the proof in three steps:
\begin{itemize}
\item [(i)] Along the chain $C_1$, all the sub-chains starting from $X_1^{(1)}$ with larger length than $S_1$ have at most as large average LPR as $S_1$, that is,  $\bar \ell_{1, i} \leq \bar \ell_{1, i_1}$, $\forall i_1 < i \leq k_1$. In terms of the procedure of Algorithm \ref{algo:equiv_algo}, all the mean LPR values in the supernodes behind $S_1$ is no larger than $\ell_{1, i_1}$, so the argument (i) holds.
\item [(ii)] Along the chain $C_1$, all the sub-chains starting from $X_1^{(1)}$ with smaller length than $S_1$ have at most as large average LPR as $S_1$, that is,  $\bar \ell_{1, i} \leq \bar \ell_{1, i_1}$, $\forall 1 \leq i < i_1$. Otherwise, suppose $i_1^{\prime} < i_1$ s.t. $\bar \ell_{1, i_1^{\prime}} > \bar \ell_{1, i_1}$. By Eq. \ref{eq:backwards_average}, we know that $\sum_{i = c}^{i_1^{\prime}} (X_i^{(1)} - \ell_{1, i_1^{\prime}}) \geq 0$, $c = 1, \ldots, i_1^{\prime}$. So to make  any supernode right behind the one ending with $X_{i_1^{\prime}}^{(1)}$ merged with its former supernode, the average LPR value of the former must be at least $\ell_{1, i_1^{\prime}}$. Thus, we can inductively conclude that $\bar \ell_{1, i_1} \geq \bar \ell_{1, i_1^{\prime}}$, which is a contradiction.
\item [(iii)] $\bar \ell_{s, i} \leq \bar \ell_{1, i_1}$, $\forall i \in \{1, \ldots, k_s\}, s \in \{1, \ldots, \nu\}$. Otherwise, wlog, suppose $\bar \ell_{2, i_2} > \bar \ell_{1, i_1}$. By Eq. \ref{eq:backwards_average}, any super node ending with $X_{i_2}^{(2)}$ has average LPR at least $\bar \ell_{2, i_2}$. Then it contradicts with the assumption that $S_1$ is the first supernode that will be merged with $R$, if $R$ has not been taken off, or by the time $S_1$ is taken off.
\end{itemize}
\end{proof}

\section{A faster implementation of HierLPR} \label{appendix:faster_algo}
In Algorithm \ref{algo:chain_merge}, we note the fact that each subchain in the tree can be partitioned into multiple blocks --- given a chain $C_r$, the breaking points are sequentially defined as
\begin{equation}
  \label{eq:breaking_points}
  p_j := \left \{
  \begin{array}{ll}
    \max_{1 \leq i \leq |C_r|} \frac{1}{C_r(i)} \sum_{k \in C_r(i)} LPR_k, & \text{~if~} j == 1\\
    \max_{p_{j - 1} < i \leq |C_r| } \frac{\sum_{k \in C_r(i)/C_r(p_{j-1})} LPR_k}{|C_r(i)| - |C_r(p_{j-1})|}, & \text{~if~} j > 1\\
  \end{array}
  \right .
\end{equation}
For example, Figure \ref{fig:faster_illustration} (i) shows a chain of $6$ nodes can be partitioned into two blocks. During the merging procedure of Algorithm \ref{algo:original_algo}, it turns out that the blocks defined by the above partitions will not be broken into smaller pieces, but can be further agglomerated. To show this, suppose there are two consecutive blocks in a chain, $B_1$, $B_2$, and $B_1$ locates ahead of $B_2$. Now we reform the blocks from the nodes in $B_1$ and $B_2$, using the rule in \eqref{eq:breaking_points}. It is obvious that nodes in $B_1$ will be clustered together. It remains to see which nodes in $B_2$ will be clustered with the nodes in $B_1$. Denote by $B_2(i)$ a sub-block consisting of the first $i$ nodes in $B_2$, and by $\ell_{B} = \frac{1}{|B|} \sum_{k \in B} LPR_k$ given a block $B$. Then, the average LPR of the nodes in $B_1$ and the first $i$ nodes in $B_2$ is computed as:
\begin{eqnarray}
  \ell_{B_1 \cup B_2(i)} &=& \frac{|B_1|\ell_{B_1} + i \ell_{B_2(i)}}{|B_1| + i} 
                         = \ell_{B_1}+\frac{\ell_{B_2(i)}-\ell_{B_1}}{|B_1|/i + 1}.\label{eq:block_LPR}
\end{eqnarray}
By the definition of block $B_2$, we have $\ell_{B_2} \geq \ell_{B_2(i)}$, $\forall i=1, \ldots, |B_2|$. If $\ell_{B_1} > \ell_{B_2}$, none of the nodes in $B_2$ will be clustered together will the nodes in $B_1$. If $\ell_{B_1} \leq \ell_{B_2}$, \eqref{eq:block_LPR} shows that all the nodes of $B_1$ and $B_2$ will form a new block. Therefore, blocks will not be broken into pieces, but can be further agglomerated. During the merging of multiple chains whose roots have the same parent, no blocks will be agglomerated since the blocks are sorted in a descending way along the merged chain; see the three descendant blocks of the bold block in Figure \ref{fig:faster_illustration} (ii). On the contrary, blocks can be agglomerated with those from the parent chain. Figure \ref{fig:faster_illustration} (iii) shows that after the chain merging, the blocks in the merged chain can be further agglomerated with the parent block (the bold one).

These observations motivate us to propose Algorithm \ref{algo:faster_algo}, a faster version of Algorithm \ref{algo:original_algo}. We avoid repeatedly computing moving averages by partitioning each chain into blocks, storing the size and the average of each block. Specifically, there are three new components we need for Algorithm \ref{algo:faster_algo}:
  \begin{itemize}
  \item \textbf{Detect breaking points.} For a chain $C_r$, breaking points can be detected by \eqref{eq:breaking_points}. Many existing algorithms can be used to this end. For example, recursion leads to an $\mathcal{O}(|C_r| \log |C_r|)$ time complexity. Figure \ref{fig:faster_illustration} (i) illustrates this step.
  \item \textbf{Merge multiple chains with defined blocks}. Merging $m$ multiple chains with detected blocks can be realized using the k-way merge algorithm. The time complexity is $\mathcal{O}(s\log m)$, where $s$ is the total number of blocks in these chains. Figure \ref{fig:faster_illustration} (ii) illustrates this step using a tree of five blocks.
  \item \textbf{Agglomerate the upstream chain and the downstream merged chain.} For a node $v \in \mathcal{P}$, denote by $C^{(v)}$ the chain ends with the node $v$. Suppose the children of $v$ as $r_1, \ldots, r_{H}$. Denote by $C_{v}$ the chain output by merging $C_{r_1}, \ldots, C_{r_H}$ using the k-way merge algorithm. Denote the blocks of $C^{(v)}$ by $B_1^{(up)}, \ldots, B_s^{(up)}$ and the blocks of $C_v$ by $B_1^{(down)}, \ldots, B_t^{(down)}$. Algorithm \ref{algo:agglo_up_down} agglomerates the blocks of $C^{(v)}$ and $C_v$ with a time complexity of $\mathcal{O}(|C^{(v)}| + |C_v|)$. Figure \ref{fig:faster_illustration} (iii) illustrates this step using the output of Figure \ref{fig:faster_illustration} (ii).
  \end{itemize}

  \begin{figure}
    \centering
    \includegraphics[width=5.5in, height = 7.2in]{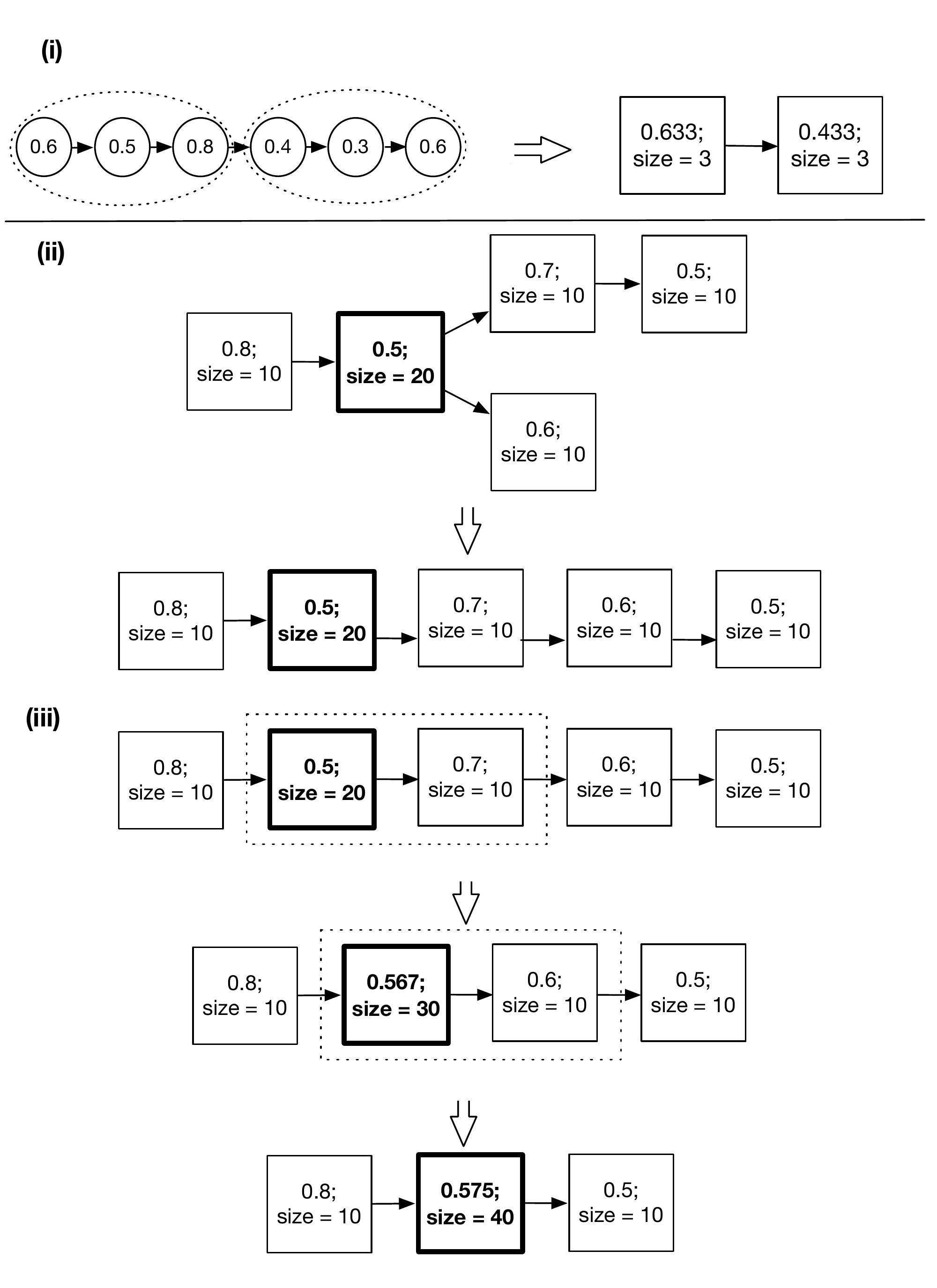}
    \caption{Illustrating the three components in Algorithm \ref{algo:faster_algo} using two examples which are separated by the solid line. The first example starts from a tree of six nodes and the second example starts from a tree of five blocks. (i) Detect breaking points of the chain of six nodes and partition them into two blocks. (ii) Merge the two children chains of the bold block. (iii) Agglomerate the upstream chain and the downstream chain around the bold node. The final list of blocks are positioned in a descending way.}
    \label{fig:faster_illustration}
  \end{figure}
  
  Throughout Algorithm \ref{algo:faster_algo}, the total time complexity consists of three parts: 1) detecting breaking points requires $\mathcal{O}(n\log K )$ computations; 2) merging multiple chains with defined blocks requires $\mathcal{O}(Dn\log K + n\log M)$ computations, where $D$ is the number of nodes that have multiple children in the graph for one instance. The quantity $D$ upper bounds the number of times each subchain merges during the algorithm; 3) agglomerating the upstream chain and the downstream merged chain requires $\mathcal{O}(Dn)$ computations. In total, the time complexity of Algorithm \ref{algo:faster_algo} is $\mathcal{O}(Dn\log K)$. In reality, most tree structures are shallow with $D < 10$. For example, the $D = 6$ and $D = 5$ in Figure \ref{disease_diag_graph1} and Figure \ref{disease_diag_graph2} respectively. So our algorithm is actually of $\mathcal{O}(n \log K)$ runtime for practical use.

\begin{algorithm}[H]
\caption{Agglomerate the blocks in the upstream chain and the downstream chain}\label{algo:agglo_up_down}
\textbf{Input: }{Blocks $B_1^{(up)}, \ldots, B_s^{(up)}$ from the upstream chain $C^{(v)}$ and Blocks $B_1^{(down)}, \ldots, B_t^{(down)}$ from the downstream chain $C_v$.}\\
\textbf{Procedure: }
\begin{algorithmic}[1]
  \STATE Let $b_0$ be $B_1^{(down)}$, $b_{-1}$ be the block ahead of it in $C^{(v)}$ and $b_{+1}$ be the block after it. Denote by $\ell_{b_0}, \ell_{b_{-1}}, \ell_{b_{+1}}$ the averaging LPR within $b_0$, $b_{-1}$ and $b_{+1}$ respectively.
  \WHILE{$\ell_{b_0} > \ell_{b_{-1}}$ or $\ell_{b_{+1}} > \ell_{b_0}$}
  \IF{$\ell_{b_0} > \ell_{b_{-1}}$}
  \STATE Agglomerate $b_0$ and $b_{-1}$. The new block is still called $b_0$ and the block ahead of it now is called $b_{-1}$.
  \ELSE
  \STATE Agglomerate $b_0$ and $b_{+1}$. The new block is still called $b_0$ and the block after it now is called $b_{+1}$.
  \ENDIF
  \ENDWHILE
\end{algorithmic}[1]
\textbf{Output: }{The new sequence of blocks}.
\end{algorithm}

\begin{algorithm}[H]
\caption{A faster implementation of the HierLPR algorithm.}\label{algo:faster_algo}
\textbf{Input: }{A forest $\mS$}\\
\textbf{Procedure: }
\begin{algorithmic}[1]
\STATE Figure out $\mathcal{P}$.
\WHILE{$\mathcal{P} \neq \emptyset$}
\STATE  Pop out a $v$ from $\mathcal{P}$. Take out all of its children $r_1, \ldots, r_{H}$. These children's descendants have at most one child. Denote the chain ends with the node $v$ as $C^{(v)}$.
\STATE  Find the breaking points $p_1^{(h)}, \ldots, p_{S_{h}}^{(h)}$ for $C_{r_h}$ by \eqref{eq:breaking_points}, $h = 1,\ldots, H$.
\STATE Merge $C_{r_1}, \ldots, C_{r_H}$, in terms of the averaging LPRs of the blocks separated by the breaking points. Denote the new chain as $C_v$.
\STATE Agglomerate blocks of $C^{(v)}$ and $C_v$ by Algorithm \ref{algo:agglo_up_down}.
\ENDWHILE
\IF{There remain multiple chains}
\STATE Mergethem use the k-way merge algorithm. 
\ENDIF
\STATE Let $\mL$ be the resulting chain.
\end{algorithmic}[1]
\textbf{Output: }{a sorted list $\mL$}.
\end{algorithm}

\section{More discussion on CSSA and Algorithm \ref{algo:equiv_algo}} \label{appendix:CSSA_discussion}
Algorithm \ref{algo:bi2011_relaxed} solves the optimization problem \eqref{opt:bi2011_relaxed} stated as below
\begin{eqnarray}
\max_{\Psi}&& \sum_{k \in \mathcal{T}} B(k) \Psi(k)\label{opt:bi2011_relaxed}\\
s.t.&& \Psi(k) \geq 0, \forall k, \quad \Psi(0) = 1, \sum_{k \in \mathcal{T}} \Psi(k) \leq L.  \nonumber\\
&&\Psi \text{ is }\mathcal{T}\text{-nonincreasing},\nonumber
\end{eqnarray}
where $\mathcal{T}$-nonincreasing means that $\Phi(k) \leq \Phi(k')$ if node $k'$ is the ancestor of node $k$. Algorithm \ref{algo:bi2011_relaxed} seems similar to Algorithm \ref{algo:equiv_algo} in spirit, and both results are almost the same except for the end behavior. However, due to the subtle distinction between the objective function \eqref{opt:bi2011_relaxed} and \eqref{eq:objfun}, the theoretical results for CSSA in \cite{bi2011} cannot be used to derive the optimality for eAUC. The specific reasoning is stated as below.

We first revisit the objective function in \eqref{opt:bi2011_original}. Note that Algorithm \ref{algo:bi2011_relaxed} has a property that $\Psi(k) = 1$ for $L$ implies $\Phi(k) = 1$ for $L'$ when $L < L'$. It implies this algorithm produces a ranking list regardless of the choice of $L$. If this property also holds for the solution to \eqref{opt:bi2011_original}, summing up these objective function values over $L = 1, \ldots, n$ maximizes \eqref{eq:objfun} (replace $B(k)$ with $LPR_k$). Unfortunately, this strategy fails on the objective function \eqref{opt:bi2011_relaxed}, that is, the weight for the $k$-th node can be either smaller or larger than $(n - k + 1)$ by summing up the solutions to \eqref{opt:bi2011_relaxed}. For example, if the last selected supernode $n(S_i)$ in Algorithm \ref{algo:bi2011_relaxed} contains $n(S_i) > L - \Gamma$, then $\Psi(S_i) = (L - \Gamma)/n(S_i) < 1$. The relaxation constraint \eqref{const:relaxed} disables the direct extension of what have been proved in \cite{bi2011} to the optimality of eAUC for Algorithm \eqref{algo:equiv_algo}. Therefore, our analysis provides a novel insight for CSSA. By showing that Algorithm \eqref{algo:original_algo} and it is equivalent to Algoirthm \eqref{algo:equiv_algo} in terms of the produced sorting list, we connect CSSA to the optimality of eAUC.

\begin{algorithm}[H]
  \caption{The CSSA algorithm}\label{algo:bi2011_relaxed}
  \textbf{Input: }A forest $\mathcal{S}$\\
  Denote $Par(S_i)$ as the parent of supernode $S_i$, $n(S_i)$ as the number of nodes in $S_i$, and $\Psi$ as a vector indicating which nodes are selected.\\
  \begin{algorithmic}[1]
    \STATE Initialize $\Psi(0) \leftarrow 1$; $\Gamma \leftarrow 1$.
    \STATE Initialize all other nodes as supernodes with $\Psi(k) \leftarrow 0$ and sort them according to the LPR value.
    \WHILE{$\Gamma < L$}
    \STATE Find $i = \argmax_i\ \frac{1}{n(S_i)} \sum_{j \in S_i} LPR_j$ 
    \IF{$\Psi(Par(S_i)) = 1$}
    \STATE  $\Psi(S_i) \leftarrow \min \{1, (L - \Gamma)/n(S_i)\}$
    \STATE  $\Gamma \leftarrow \Gamma + n(S_i)$
    \ELSE
    \STATE Condense $S_i$ and $Par(S_i)$ as a new supernode.
  \ENDIF
  \ENDWHILE
\end{algorithmic}
\textbf{Output: }Vector $\Psi$.
\end{algorithm}

\section{Experiments}\label{appendix:experiments}
\subsection{Background of real data}\label{appendix:real_data_background}
\cite{huang2010} developed a classifier for predicting disease along the Unified Medical Language System (UMLS) directed acyclic graph, trained on public microarray datasets from the National Center for Biotechnology Information (NCBI) Gene Expression Omnibus (GEO). GEO was originally founded in 2000 to systematically catalog the growing volume of data produced in microarray gene expression studies. % The large majority of data on GEO are these studies, and they continue to be the most common kind of study submitted today \citep{barrett2013}.
GEO data typically comes from research experiments where scientists are required to make their data available in a public repository by a grant or journal guidelines. At of July 2008, GEO contained 421 human gene expression studies on the three microarray platforms that were selected for analysis (Affymetrix HG-U95A (GPL91), HG-U133A (GPL96), and HG-U133 Plus 2 (GPL570)). Briefly, 100 studies were collected, yielding a total of 196 datasets. These were used for training the classifier in \cite{huang2010}. 

Labels for each dataset were obtained by mapping text from descriptions on GEO to concepts in the UMLS, an extensive vocabulary of concepts in the biomedical field organized as a directed acyclic graph. The mapping resulted in a directed acyclic graph of 110 concepts matched to the 196 datasets  at two levels of similarity -- a match at the GEO submission level, and a match at the dataset level, with the latter being a stronger match. The disease concepts and their GEO matches are listed in Table S2 in the supplementary information for \cite{huang2010}. 

Training a classifier for each label is a complex multi-step process, and is described in detail in the Supplementary Information of \cite{huang2010}. We summarize that process here. In the classifier for a particular disease concept, the negative training instances were the profiles among the 196 that did not match with that disease concept. The principal modeling step involved expressing the posterior probability of belonging to a label in terms of the log likelihood ratio and some probabilities that have straightforward empirical estimates. The log likelihood ratio was modeled with a log-linear regression. A posterior probability estimate was then obtained for each of the 110 $\times$ 196 instances in the data by leave-one-out cross-validation, i.e. estimating the $i$-th posterior probability based on the remaining ones. These posterior probability estimates were used as the first-stage classifier scores. An initial label assignment for the first-stage was then obtained by finding the optimal score cutoffs for each classifier.

\subsection{Characteristics of the disease diagnosis data and hierarchy}\label{appendix:char_real_data}
The full hierarchy graph is partitioned into two parts as shown in Figures \ref{disease_diag_graph1} and \ref{disease_diag_graph2} respectively. As the two figures show, the 110 nodes are grouped into 24 connected sets. In general, the graph is shallow rather than deep: the maximum node depth is 6, though the median is 2. Only 10 nodes have more than one child. This occurs because 11 of the connected sets are standalone nodes, while six are simple two-node trees. The two largest sets consist of 28 and 30 nodes, respectively.

\begin{figure}[h!]
\centering
\caption{Structure of the disease diagnosis dataset, part 1 of 2. The colors correspond to node quality: white indicates that a node's base classifier has AUC between $(0.9, 1]$; light grey, $(0.7, 0.9]$, dark grey, $(0, 0.7]$. The values inside the circles indicate the number of positive cases, while the value underneath gives the maximum percentage of cases shared with a parent node.}
\label{disease_diag_graph1}
\includegraphics[width=4.5in]{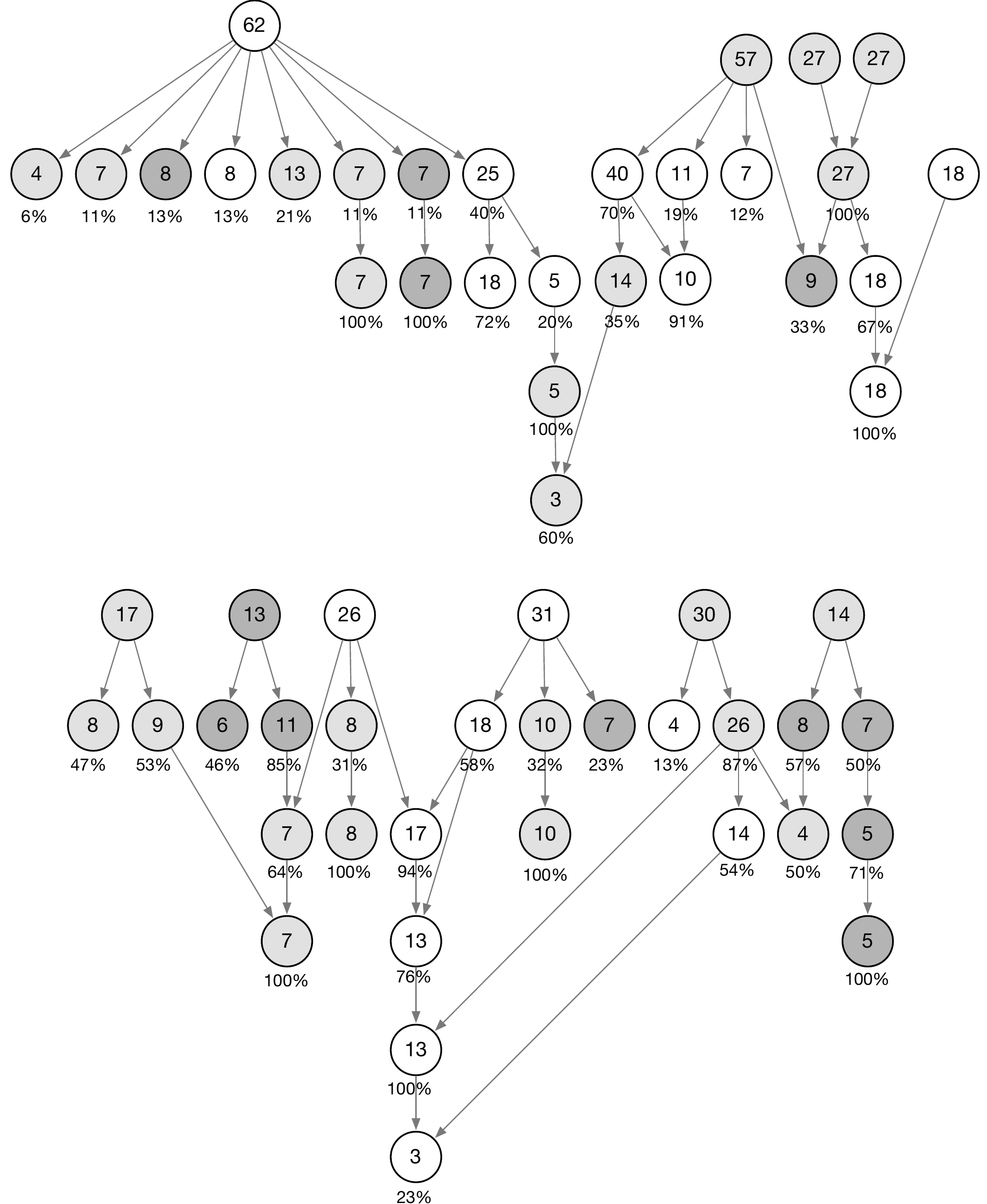}
\end{figure}

\begin{figure}[h!]
\centering
\caption{Structure of the disease diagnosis dataset, part 2 of 2. The colors correspond to node quality: white indicates that a node's base classifier has AUC between $(0.9, 1]$; light grey, $(0.7, 0.9]$, dark grey, $(0, 0.7]$. The values inside the circles indicate the number of positive cases, while the value underneath gives the maximum percentage of cases shared with a parent node.}
\label{disease_diag_graph2}
\includegraphics[width=4.5in]{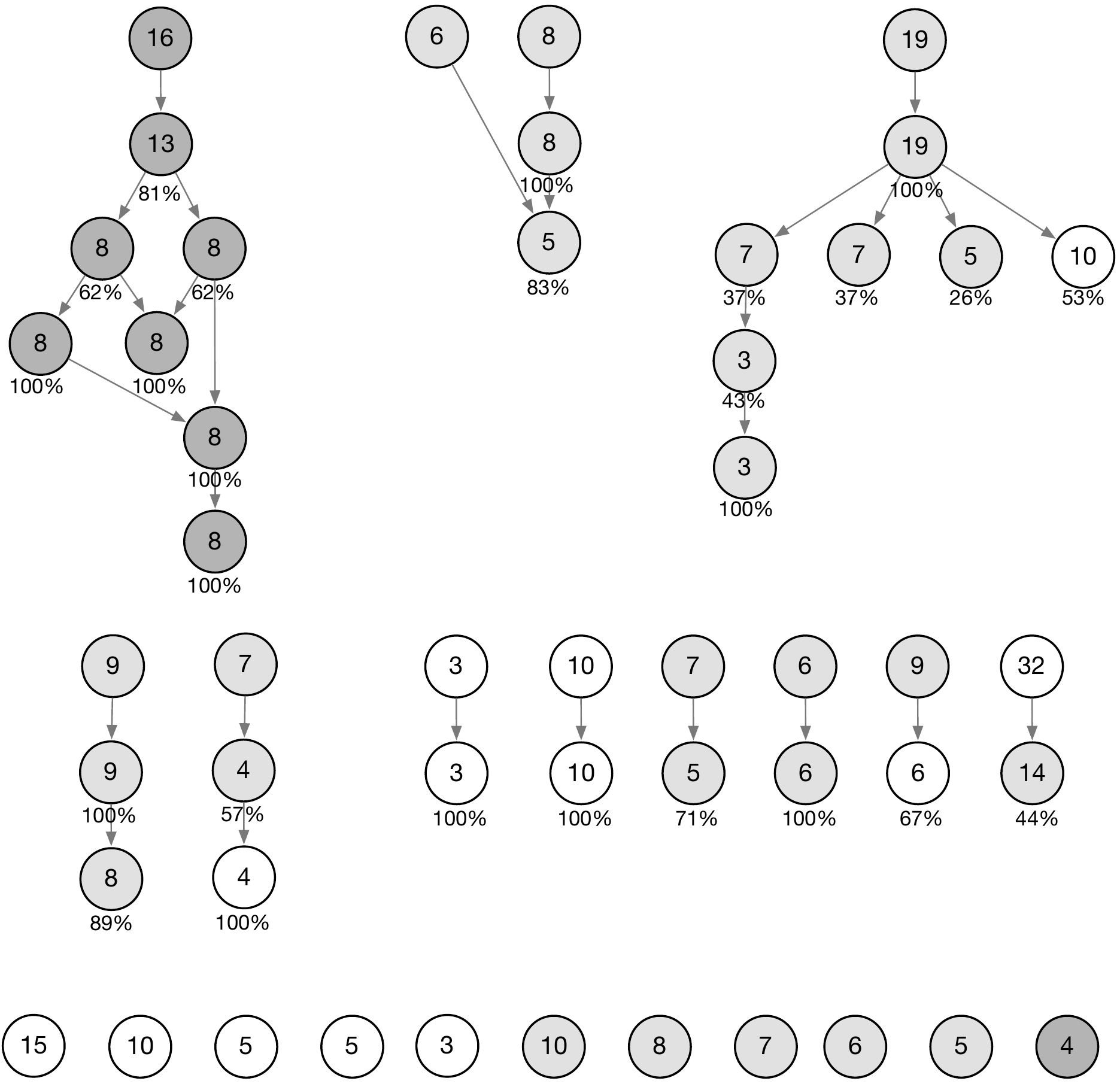}
\end{figure}
The graph nearly follows a tree structure. Most nodes have only one parent or are at the root level. Only 15 nodes have 2 parents, and 2 nodes have 3 parents. Most nodes do not have a high positive case prevalence. The largest number of samples belonging to a label is 62, or a 32.63\% positive case prevalence. The average prevalence is 5.89\%, with a minimum prevalence of 1.53\%, corresponding to 3 cases for a label. Data redundancy occurs as an artifact of the label mining: usually, the positive instances for a disease concept are the same as for its ancestors. There are few datasets that are tagged with a general label and not a leaf-level one. Twenty six nodes or 23.64\% of all nodes share the same data as their parents, so they have the same classifier, and therefore the same classifier scores or $LPR$s as their parents. If we take the number of nodes that share more than half of their data with their parent, this statistic rises to 50\%. A consequence of this redundancy is that the graph is shallower than appears in the figure: for example, the first connected set in the top left of Figure \ref{disease_diag_graph2} appears to have six levels, but actually only has three because the last three levels do not contain any new information.

\subsection{Details of ClusHMC implementation}\label{appendix:clusHMC_details}
We use ClusHMC and follow \cite{dimitrovski2011} by constructing bagged ensembles and used the original settings of \cite{vens2008}, weighting each node equally when assessing distance, i.e. $w_i = 1$ for all $i$. In addition to node weights, the minimum number of instances is set to 5, and the minimum variance reduction is tuned via 5-fold cross validation from the options 0.60, 0.70, 0.80, 0.90, and 0.95. Following the implementation of \cite{lee2013}, a default of 100 PCTs are trained for each ClusHMC ensemble; each PCT is estimated by resampling the training data with replacement and running ClusHMC on the result.

\subsection{Complete results on precision recall curves of the simulation study}\label{appendix:more_results}
%In this section, we present four more simulation results on the performance of HierLPR. In Figure \ref{class_sim_PR2}, Setting 1 has the same structure as the correponding one in Figure \ref{class_sim_settings}, but with conditional probability fixed at 0.95 for each node, so that the effect of strong dependence between nodes could be examined. Setting 8 has the same structure as setting 4 but all nodes of high quality. Setting 9 and 10 have the same structure as setting 1 but all nodes of high quality, and setting 10 with conditional probability fixed at 0.95 for each node. Setting 11 has the same structure as setting 6 but all nodes of high quality.
Figure \ref{class_sim_PR2} depicts the precision-recall curves for settings 1 through 4 of the normal range --- between 0 and 1.

% \begin{figure}[h!]
%   \caption{Precision recall curves comparing ClusHMC, HierLPR, and LPR for settings 1 through 4. The recall rates are between 0 and $1$.}
%     \label{class_sim_PR2}
%   \begin{minipage}{1.0\textwidth}
%     \includegraphics[width=\textwidth]{sim_supp1_aug.pdf}
% %    \subcaption{Favorable cases.}
%   \end{minipage}\\
%   \begin{minipage}{1.0\textwidth}
%     \includegraphics[width=\textwidth]{sim_supp2_aug.pdf}
% %    \subcaption{Unfavorable cases.}
%   \end{minipage}
% \end{figure}

\begin{figure}[h!]
  \caption{Precision recall curves comparing ClusHMC, HierLPR, and LPR for settings 1 through 4. The recall rates are between 0 and $1$.}
    \label{class_sim_PR2}  
    \includegraphics[width=1\textwidth]{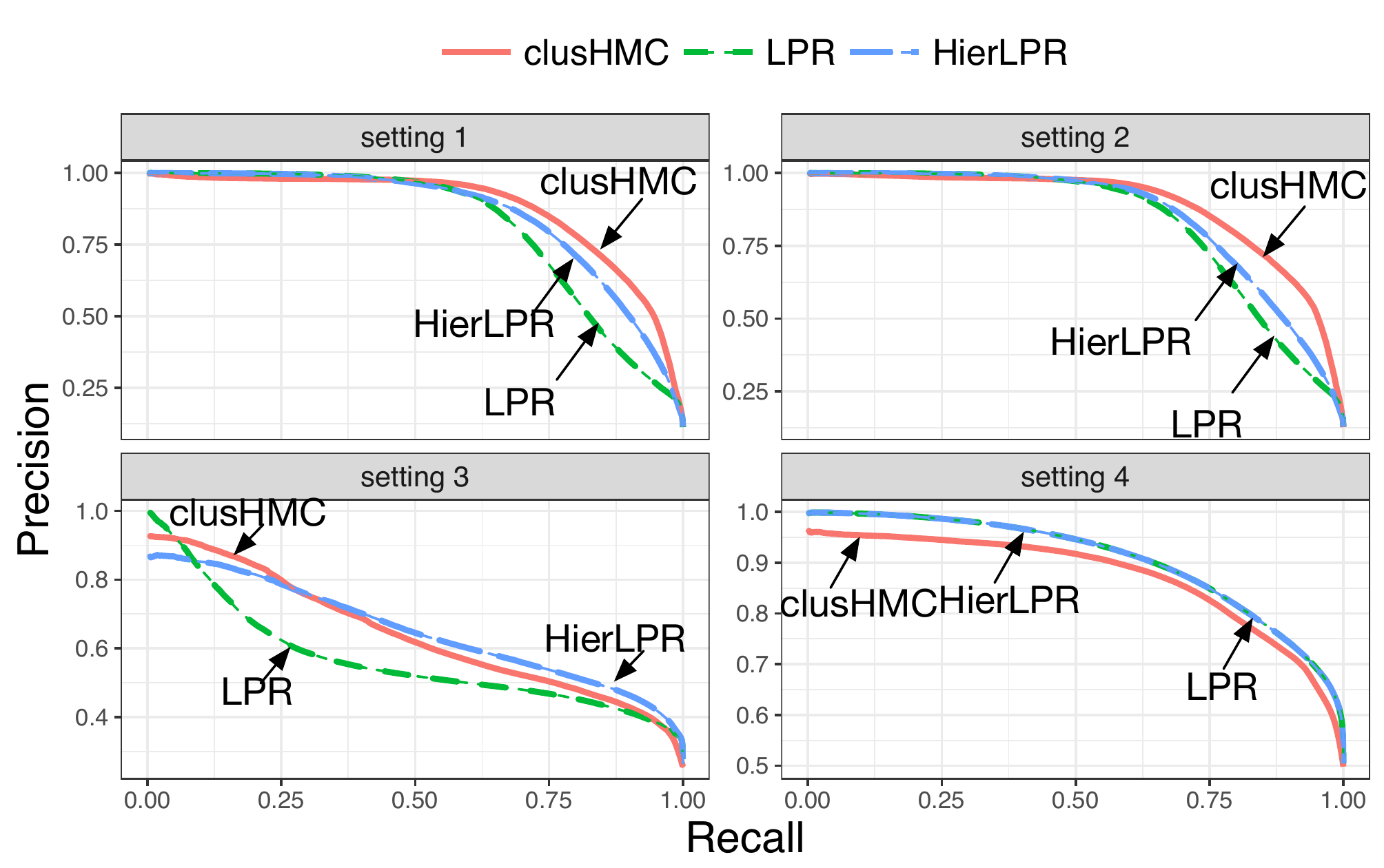}
\end{figure}

\end{appendices}

\clearpage

\bibliography{hierLPR}
\bibliographystyle{apalike}

\end{document}